\documentclass{article}

\usepackage{arxiv}
\usepackage[numbers,sort&compress]{natbib}
\usepackage{doi}

\usepackage{iftex}
\ifPDFTeX
  \usepackage[utf8]{inputenc}
  \usepackage[T1]{fontenc}
\else
  \usepackage{fontspec}
  \defaultfontfeatures{Ligatures=TeX}
\fi
\usepackage[english]{babel}

\usepackage{amsmath,amssymb,amsfonts,amsthm,mathtools}
\usepackage{booktabs}
\usepackage{colortbl}
\usepackage{graphicx}
\usepackage{microtype}
\usepackage{multirow}
\usepackage{nicefrac}
\usepackage{xcolor}

\usepackage[font={small},labelfont=bf]{caption}
\usepackage[nameinlink]{cleveref}
\usepackage{algorithm}
\usepackage{algorithmic}
\usepackage{placeins}
\usepackage{xspace}

\newtheorem{theorem}{Theorem}[section]
\newtheorem{lemma}[theorem]{Lemma}
\newtheorem{proposition}[theorem]{Proposition}
\newtheorem{corollary}[theorem]{Corollary}
\newtheorem*{proposition*}{Proposition}
\theoremstyle{definition}
\newtheorem{definition}[theorem]{Definition}
\newtheorem{assumption}[theorem]{Assumption}
\theoremstyle{remark}
\newtheorem{remark}[theorem]{Remark}

\graphicspath{{./}{figures/}}

\newcommand{\sys}{{\normalfont\textsc{FlowForge}}\xspace}

\newcommand{\buildfigurewidth}[1]{0.8\textwidth}

\AtBeginEnvironment{thebibliography}{%
  \raggedright
  \sloppy
  \emergencystretch=3em\relax
}

\title{\sys: A Staged Local Rollout Engine for Flow Field Prediction}
\author{
  Xiaowen Zhang\thanks{These authors contributed equally.} \\
  Shanghai Jiao Tong University \\
  Shanghai, China \\
  \texttt{handshaker@sjtu.edu.cn}
  \And
  Ziming Zhou\footnotemark[1] \\
  University of Michigan \\
  Ann Arbor, USA \\
  \texttt{zimingzh@umich.edu}
  \And
  Fengnian Zhao \\
  Shanghai Jiao Tong University \\
  Shanghai, China \\
  \texttt{iclover@sjtu.edu.cn}
  \And
  David L.S. Hung \\
  Shanghai Jiao Tong University \\
  Shanghai, China \\
  \texttt{dhung@sjtu.edu.cn}
}
\date{}

\renewcommand{\headeright}{arXiv Preprint}
\renewcommand{\undertitle}{arXiv Preprint}
\renewcommand{\shorttitle}{\textit{\sys}: Staged Local Rollout for Flow Field Prediction}

\hypersetup{
  pdftitle={\sys: A Staged Local Rollout Engine for Flow Field Prediction},
  pdfauthor={Xiaowen Zhang, Ziming Zhou, Fengnian Zhao, David L.S. Hung},
  colorlinks=true,
  linkcolor=blue,
  citecolor=blue,
  urlcolor=blue
}

\begin{document}
\maketitle

\begin{abstract}
Deep learning surrogates for CFD flow field prediction often rely on large,
complex models, which can be slow and fragile when data are noisy or incomplete.
We introduce \textbf{\sys}, a staged local rollout engine that predicts future
flow fields by compiling a locality-preserving update schedule and executing it
with a shared lightweight local predictor. Rather than producing the next frame
in a single global pass, \sys rewrites spatial sites stage by stage so that each
update conditions only on bounded local context exposed by earlier stages. This
compile--execute design aligns inference with short-range physical dependence,
keeps latency predictable, and limits error amplification from global mixing.
Across PDEBench, CFDBench, and BubbleML, \sys matches or improves upon strong
baselines in pointwise accuracy, delivers consistently better robustness to noise
and missing observations, and maintains stable multi-step rollout behavior while
reducing per-step latency.

\end{abstract}

\section{Introduction}

Flow field prediction is a core operation in time-dependent
computational fluid dynamics (CFD): given a short history window $U_{t-m+1:t}$, engineers and
experimentalists repeatedly need a fast surrogate for a future state $U_{t+k}$
to explore small changes in operating conditions, boundary settings, or
geometry. \cite{ml-fluid-mechanics} In this workflow, the surrogate ML model is only useful if it is 
(i) accurate under nominal conditions,
(ii) reliable under realistic input imperfections (noise, missing regions, imperfect
boundary conditions \cite{piv-uq}),
and (iii) fast enough for interactive iteration \cite{deeponet}.

Current ML surrogates yet do not provide a reliable off-the-shelf solution. In practice, users must
choose among many families (e.g., U-Net \cite{unet}, Fourier Neural Operator (FNO) \cite{fno},
and other physics-informed neural networks (PINNs) \cite{pinns}) whose performance and stability can vary
substantially across datasets and noise regimes.
Empirical studies in Section~\ref{sec:robustness} show that even strong baselines degrade under
realistic input corruptions and incur higher per-step latency. 

A common root cause is architectural. Many popular models rely on deep encoder--decoder
pipelines or layers with global receptive fields, such as Fourier spectral operators 
or global attention, which mix information across the entire spatial grid in a single update. 
While regularization or spectral truncation can mitigate some instabilities, 
they do not alter the underlying information-flow graph: every output location is 
still computed in one global pass and can be influenced by defects anywhere in the input.

This global structure is not incidental but induced by the standard supervised objective, 
which trains the model to predict the entire future field in one global update,
\[
U_{t+k} = f_k(U_{t-m+1:t}).
\]
Over short horizons, however, physical influence is predominantly local: 
each spatial site depends almost only on a bounded neighborhood. 
Single-pass global predictors instead compute all sites simultaneously, 
allowing information from distant regions to affect every update.
As a result, local defects spread across the domain and inference needlessly pays 
for global interactions (\autoref{fig:intro_flowforge}).

\noindent Our key idea is to predict the next flow field not in a single global pass, 
but as a \textbf{local rollout}. Instead of updating all spatial locations 
simultaneously, we generate the single future state $U_{t+k}$ by progressively 
rewriting spatial sites in a locality-preserving order. At each step, a site is 
predicted using (i) the past history $U_{t-m+1:t}$ and (ii) nearby future values 
that have already been updated. In this way, information propagates locally and 
incrementally across the grid, rather than being mixed globally in one shot.

Formally, this corresponds to factorizing the conditional distribution of 
the next field into ordered, neighborhood-conditioned site-wise updates:
\begin{equation}
p(U_{t+k}\mid U_{t-m+1:t})
=
\prod_{i=1}^{N}
p\!\left(
u^{(\sigma(i))}_{t+k}
\,\middle|\,
U_{t-m+1:t},
\,u^{(\sigma(<i))}_{t+k}
\right),
\label{eq:intro_factorization}
\end{equation}
where $\sigma$ defines an order over the $N$ spatial sites, and $u^{(\sigma(<i))}_{t+k}$ 
denotes values generated earlier in that order. 
The history window $U_{t-m+1:t}$ is channel-stacked as $\mathrm{concat}(U_{t-m+1},\ldots,U_t)\in\mathbb{R}^{N\times (mc)}$.

\begin{figure}[t]
    \centering
    \includegraphics[width=\buildfigurewidth{0.8\columnwidth}]{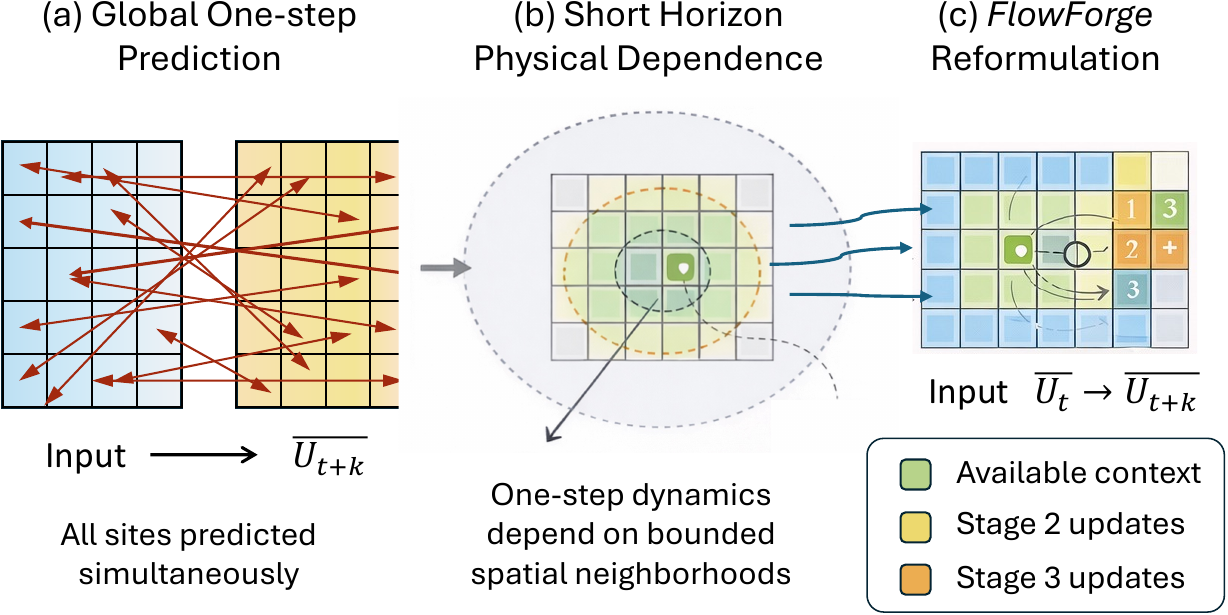}
    \caption{
    From global one-step prediction to locality-preserving rollout.
    (a) Global predictors update all sites simultaneously.
    (b) Short-horizon dynamics exhibit bounded spatial dependence.
    (c) \sys performs a staged local rollout aligned with this structure.
    }

    \label{fig:intro_flowforge}
\end{figure}

\noindent Na\"ively implementing \eqref{eq:intro_factorization} is impractical:
a strict site-by-site rollout would create a long sequential dependency chain,
severely limiting parallelism and amplifying errors across the grid. This raises
a central question: \emph{can we preserve the local physical dependence of the
update while avoiding a fully sequential execution?} In other words, how can we
maintain locality---so each site depends only on nearby values---yet still achieve
fast, predictable inference at scale?

\noindent \sys resolves this tension through a \emph{compile--execute} design.
The \emph{compiler} precomputes a rollout plan that (i) follows a
locality-preserving traversal order, (ii) partitions sites into a small number
of \emph{stages} that can be updated in parallel, and (iii) assigns each site a
bounded neighborhood drawn only from earlier stages. At inference time, the
\emph{executor} performs a stage-by-stage rewrite of the field, applying the
same lightweight local predictor to all sites within a stage in parallel while
conditioning strictly on the plan-defined context. In effect, the compiler
fixes \emph{how} information is allowed to flow across the domain, and the
executor simply executes this structured schedule.

As a result, information and errors propagate only locally rather than globally,
parallelism is exposed in a controlled, stage-wise manner, and each site is
updated using a fixed-size neighborhood. Together, these choices make inference
both hardware-friendly and latency-predictable, even as problem size grows.

The efficient local-rollout design and implementation make \sys fast to run and
resilient to input imperfections. \sys achieves best-or-second-best pointwise accuracy on
10/11 datasets,
while keeping multi-step rollouts stable. Under additive noise and block or
edge-localized masking corruptions, \sys exhibits the smallest increase in
error, particularly near boundaries, and maintains a nearly flat per-point
latency as resolution grows. These results suggest that aligning the prediction
task with short-range physical information flow enables small, stable models to
outperform complex architectures in realistic conditions.

\paragraph{Contributions.}
(1) \textbf{Formulation:} We formulate next-step flow prediction as a
\emph{local rollout}, factorizing
$p(U_{t+k}\!\mid\!U_{t-m+1:t})$ into ordered, neighborhood-conditioned spatial updates.
(2) \textbf{System:} We present \sys, a compile--execute system that enables efficient
and stable training and inference under this formulation.
(3) \textbf{Evaluation:} We empirically validate \sys on CFDBench, PDEBench, and
BubbleML, covering accuracy, rollout behavior, robustness, latency, and ablations.

\section{Staged Local Rollout}
\label{sec:theory}

A one-step update in time-dependent CFD does not depend on the entire spatial
field. Physical transport is constrained by finite propagation speed, and
diffusion spreads information over a bounded distance within a single time step.
Nevertheless, many learned surrogates either predict all locations
simultaneously---implicitly assuming global dependence---or rely on deep
receptive fields, which increases latency and degrades robustness under
perturbations.

Our goal is to realize the spatial autoregressive formulation in
\eqref{eq:intro_factorization} in a way that makes this locality explicit while
retaining parallelism. We do so by predicting the next field through a small
number of sequential stages, where each stage writes a subset of locations in
parallel and later stages are allowed to condition on values written earlier.

\subsection{Staged formulation of local rollout}
\label{sec:rollout_formulation}

Rather than predicting all $N$ spatial locations at once, we introduce an
ordering $\sigma$ over locations and partition it into $K$ disjoint stages
$S_1,\ldots,S_K$. The stages are executed sequentially, but all locations within
a stage are updated in parallel. During this process, we maintain a working
array $Y \approx U_{t+1}$ that is progressively filled.

At stage $s$, the predictor produces values for all locations in $S_s$ using the
current field $U_t$ and the values already written in earlier stages:
\begin{equation}
Y_{S_s} \;=\; G_\theta\!\big(U_t,\;Y_{S_{<s}}\big), \qquad
S_{<s} := \{S_1,\ldots,S_{s-1}\}.
\label{eq:stage_map}
\end{equation}
The update is masked in the sense that predictions at stage $s$ may read
$U_t$ and $Y_{S_{<s}}$, but never entries written in the same or later stages.
Executing stages $1{:}K$ yields a complete prediction $Y$.

This construction reduces the effective sequential dependency depth from $N$
locations to $K \ll N$. The choice of stages therefore directly controls how much
parallelism is exposed and how much predicted context each location can access.
In practice, this choice should reflect the finite spatial extent of one-step
physical dependencies.

\subsection{Physical dependency structure and stage design}
\label{sec:rollout_structure}

We exploit this locality by appealing to the finite spatial extent of one-step
physical dependencies. For instance, in an advection--diffusion system over a
time step $\Delta t$, information is transported along characteristics at speed
at most $c_{\max}$ and diffused over a radius proportional to
$\sqrt{\nu \Delta t}$. Combining these effects yields a conservative bound on the
one-step dependency radius,
\[
r_{\mathrm{req}} \;\approx\; c_{\max}\Delta t \;+\; r_p,
\]
where $r_p$ captures the effective diffusion radius
(Proposition~\ref{prop:cfl}). Beyond this distance, the influence of the current state
on a local update is negligible.

On a uniform grid with spacing $\Delta x$, this implies that predicting a
location only requires access to approximately
\[
R \;\approx\; \left\lceil \tfrac{r_{\mathrm{req}}}{\Delta x} \right\rceil
\]
neighboring rings or shells. A sufficient condition for physical consistency is
therefore that, when a location is updated, earlier stages have already covered
most of this neighborhood (Lemma~\ref{lem:ring-coverage}). This naturally motivates
locality-preserving traversal orders and outward-expanding stage schedules, in
which later stages condition on progressively larger spatial context.

\subsection{Stability and latency implications}
\label{sec:rollout_stability_latency}

With these schedule principles in mind, we next quantify what staging implies
for stability and end-to-end latency. Because later stages condition on a more
complete neighborhood, uncertainty tends to decrease as rollout progresses. If
the stage-restricted map induced by $G_\theta(U_t,\cdot)$ is Lipschitz with constant $L_s$, then a
perturbation introduced at stage $r$ can affect the output of stage $s>r$ by at
most $\prod_{q=r+1}^{s} L_q$ (Proposition~\ref{prop:prodL}). This highlights the benefit
of keeping the number of stages $K$ small and allocating more locations to later
stages, where richer context is available.

Staging also makes latency explicit. In an ideal parallel execution model, if
updating one location costs $t_{\mathrm{loc}}$ and stage $s$ processes $|S_s|$
locations with parallel width $w_s$, the total makespan is
\[
T_{\mathrm{total}}
=\sum_{s=1}^{K}\Big\lceil \tfrac{|S_s|}{w_s}\Big\rceil\,t_{\mathrm{loc}}
\qquad(\text{Proposition}~\ref{prop:makespan}).
\]
The stage schedule therefore provides a principled and tunable trade-off between
conditioning strength, stability, and end-to-end latency.

\section{\sys Design}
\label{sec:methods}

\noindent Section~\ref{sec:theory} motivates predicting $U_{t+1}$ via a \emph{staged local
rollout}: we maintain a working field and overwrite it in stages, with the rule
that predictions in stage $s$ may depend only on values written in stages
$<s$. We adopt this formulation because it makes the dependency structure
explicit and controllable. In particular, it (i) guarantees that within-stage
updates can be parallelized without circular dependencies, (ii) forces each
prediction to use a bounded local neighborhood consistent with the required
one-step dependency radius, and (iii) exposes direct accuracy--latency controls
through the number of stages and the context size. \sys implements this idea as
a compile--execute system: an offline compiler fixes an ordering, stage
boundaries, and per-site context indices; an online executor follows these
tables to produce $\widehat{U}_{t+1}$.

\vspace{0.25em}
\subsection{Overview}
\label{sec:overview}

\noindent \sys separates \emph{planning} from \emph{execution}. The plan
determines what information is visible when predicting each site: it specifies
(i) a traversal order $\sigma$ over spatial sites, (ii) a partition of the
resulting sequence into stages $S_{1:K}$, and (iii) for each site, a fixed list
of at most $H$ earlier indices that may be read as context. The executor is then
a deterministic staged rewrite of a working buffer: for stage $s$, it gathers
the plan-specified context for all $i\in S_s$, applies a shared local predictor
$G_\theta$ in batch, and writes the results back. Because the plan only allows
reads from earlier stages, causality is enforced by construction and no runtime
masking checks are required.

\noindent This structure positions the remaining subsections. Section~\ref{sec:executor}
formalizes the executor given a plan. Section~\ref{sec:compiler} describes how the plan
and lookup tables are synthesized to balance neighborhood coverage against
sequential depth. Section~\ref{sec:training} aligns training with the same plan-defined
inputs to avoid rollout mismatch. Section~\ref{sec:scalability} extends the same
mechanism to high-resolution settings.

\begin{figure}
    \centering \includegraphics[width=\buildfigurewidth{\linewidth}]{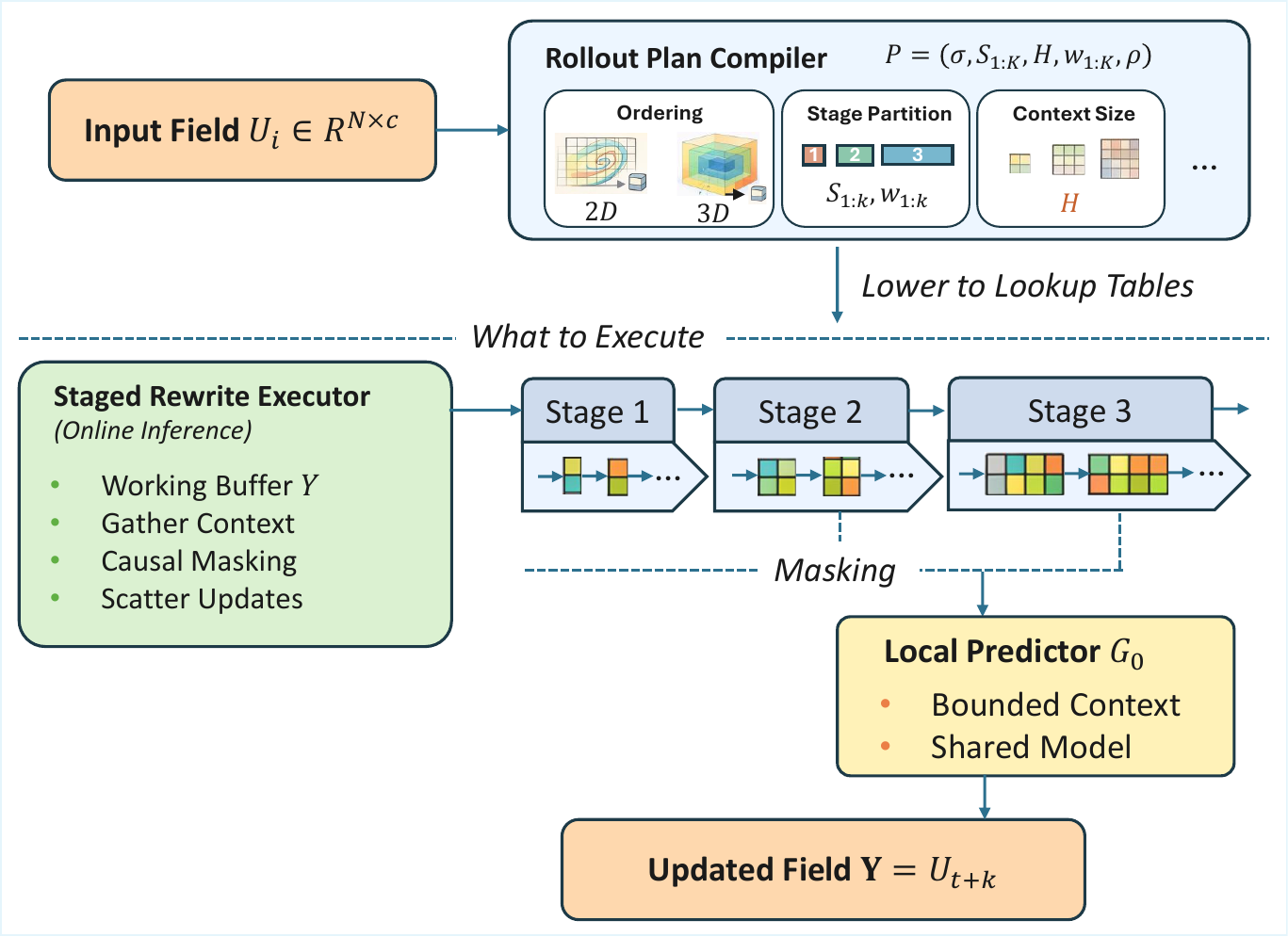}
    \caption{\sys workflow. Offline, we compile a rollout plan and lower it into
    index tables. Online, the executor overwrites a working buffer stage by stage
    using a shared local predictor $G_\theta$, producing $\widehat{U}_{t+1}$.}
    \label{fig:framework}
\end{figure}

\vspace{0.25em}
\subsection{Staged Rewrite Executor}
\label{sec:executor}

\noindent At inference time, the executor maintains a working buffer in the
compiled rollout order and rewrites it in $K$ stages. Let $\sigma$ be the
compiled ordering of spatial sites, and let
$\widetilde{U}_t := \mathrm{concat}(U_{t-m+1},\ldots,U_t)\in\mathbb{R}^{N\times (mc)}$
be the channel-stacked history window (we write $\widetilde{U}_t$ as $U_t$
below). The buffer is a sequence
\[
Y \;=\; (y_1,\ldots,y_N),\qquad y_i \in \mathbb{R}^c,
\]
where $y_i$ corresponds to site $\sigma(i)$ and is initialized from the current
field,
\[
y_i \leftarrow u_t^{(\sigma(i))}\qquad \text{for } i=1,\ldots,N,
\]
which is subsequently overwritten into the prediction $\widehat{U}_{t+1}$.

\noindent The rollout plan partitions indices $1{:}N$ into disjoint stages
$S_1,\ldots,S_K$ executed sequentially. Each stage updates all sites in $S_s$
in parallel, subject to a causality constraint: a site in stage $s$ may only
read values written in earlier stages. For each $i\in S_s$, the compiler
constructs a context set $\mathrm{ctx}(i)$ satisfying
\[
\mathrm{ctx}(i) \subseteq \bigcup_{r=1}^{s-1} S_r.
\]
The plan is lowered into index tables so the executor can gather context without
runtime decisions. At stage $s$, the executor evaluates, for all $i\in S_s$,
\begin{equation}
\label{eq:executor_stage}
  \widehat{y}_i \;=\; G_\theta\!\Big(\Phi(U_t,\sigma(i)),\;\Psi(Y;\,i),\;\mathrm{pos}(\sigma(i))\Big),
\end{equation}
and then commits the results by writing them back to the working buffer:
\begin{equation}
\label{eq:executor_scatter}
y_i \leftarrow \widehat{y}_i,\qquad i\in S_s.
\end{equation}

\noindent We use a single shared local predictor $G_\theta$ across all stages:
the update rule is stationary, and stages differ only in how much previously
written neighborhood is exposed by $\Psi(\cdot)$. The causality constraint
ensures correctness of within-stage parallelism: because every $\Psi(Y;i)$
indexes only $\cup_{r<s}S_r$, all context values are already written and remain
fixed while stage $s$ is computed, so no prediction in $S_s$ depends on another
prediction in $S_s$. Consequently, the executor can batch all sites in $S_s$
and evaluate~\eqref{eq:executor_stage} in parallel without circular
dependencies or runtime masking.

\noindent Operationally, each stage follows a three-step loop:
(1) gather a bounded set of context vectors from $Y$ using the precomputed
indices, (2) apply $G_\theta$ to the resulting batch, and (3) scatter the
outputs back into $Y$. This exposes two direct runtime controls: increasing the
number of stages $K$ increases sequential depth, while increasing the context
budget $H$ increases available information per update at the cost of additional
memory traffic. Because both quantities are fixed by the plan, inference cost is
predictable.

\subsection{Rollout Plan Compiler}
\label{sec:compiler}

\noindent After choosing the executor, the main design question is how to choose a plan
that gives each prediction enough local information without making the rollout
too sequential. The compiler outputs
\[
\mathcal{P} \;=\; (\sigma,\; S_{1:K},\; H,\; w_{1:K})
\]
Here $\sigma$ is the order of sites, $S_{1:K}$ is the stage partition in that
order, $H$ is the maximum number of context sites used by each prediction, and
$w_s$ is the batch width used when executing stage $s$. Optionally, we can also apply a coarse-to-fine strategy with factor $\rho$ to downscale the grid
and upsample back to the native grid (Section~\ref{sec:scalability}); we treat $\rho$ as an execution option rather than part of $\mathcal{P}$.

\noindent We implement the compiler in two steps. First, we choose $\sigma$ and
$S_{1:K}$ to make already-written neighbors available at decode time.
Section~\ref{sec:theory} motivates a required dependency radius $r_{\mathrm{req}}$; on
a grid with spacing $\Delta x$ this corresponds to roughly
$R \approx \lceil r_{\mathrm{req}}/\Delta x \rceil$ rings (2D) or shells (3D).
Intuitively, we want a site to be decoded after many sites in its $R$-ring
neighborhood have already been written. Second, once $\sigma$ and $S_{1:K}$ are
fixed, we precompute simple index tables so the executor can read context and
write outputs without runtime decisions.

\noindent In this work we use an Outward Spiral traversal: starting from a seed location, we visit spatially adjacent sites
while expanding the visited region outward. This ordering is easy to implement,
keeps nearby sites close in the sequence, and naturally increases the fraction
of already-written neighbors as the rollout progresses. We then use a simple,
predictable stage template (equal-sized by default), which makes $K$ a direct
knob for latency.

\begin{algorithm}[t]
\caption{Two-step compiler for \sys}
\label{alg:compiler}
\begin{algorithmic}[1]
\REQUIRE Grid sites $\mathcal{V}$ (size $N$), spacing $\Delta x$, required radius $r_{\mathrm{req}}$,
number of stages $K$ (or stage size), context budget $H$
\ENSURE Rollout plan $\mathcal{P}$ and index tables $\mathcal{T}$
\STATE $R \gets \lceil r_{\mathrm{req}}/\Delta x \rceil$
\STATE Choose traversal $\sigma$ (Outward Spiral from seed)
\STATE Partition $1{:}N$ into stages $S_1,\ldots,S_K$ (e.g., equal-sized)
\STATE Build stage tables $(\texttt{stage\_ptr},\texttt{stage\_idx})$ for iterating stage indices
\FOR{$i=1$ to $N$}
  \STATE Find candidate spatial neighbors of site $\sigma(i)$ within $R$ rings/shells
  \STATE Let $s$ be the stage containing $i$ (i.e., $i \in S_s$)
  \STATE Keep up to $H$ neighbors $j$ such that $j < i$ \textbf{and} $j \in \cup_{r=1}^{s-1} S_r$ (i.e., $j$ is in an earlier stage)
  \STATE Store them in $\texttt{ctx\_idx}[i,1{:}H]$ (and a mask if fewer than $H$ exist)
\ENDFOR
\STATE Build per-site metadata tables for $\mathrm{pos}(\cdot)$ and $\Phi(\cdot)$
\STATE \textbf{return} $\mathcal{P}=(\sigma,S_{1:K},H,w_{1:K})$ and $\mathcal{T}$
\end{algorithmic}
\end{algorithm}

\noindent The resulting executor runtime can be estimated under the same
idealized model as Section~\ref{sec:theory}. Because stages run sequentially, total
time is approximately
\[
T_{\mathrm{total}}(\mathcal{P}) \;=\;
\sum_{s=1}^{K}\Big\lceil \tfrac{|S_s|}{w_s}\Big\rceil\,t_{\mathrm{loc}}(H),
\]
where $t_{\mathrm{loc}}(H)$ summarizes the per-batch cost of reading $H$ context
vectors and running $G_\theta$. 

This analytical model defines the feasible search space of configurations satisfying a latency budget $T_{\mathrm{budget}}$. Within this feasible set, we employ a physics-informed selection strategy to identify the optimal plan $\mathcal{P}^*$. Specifically, we execute candidate configurations on a validation subset and minimize a composite objective:
\begin{equation*}
    \mathcal{L}_{\text{total}}(\mathcal{P}) = \sum_{m \in \mathcal{M}} \lambda_m \frac{\mathcal{E}_m(\mathcal{P})}{\mu_m}.
\end{equation*}
Here, $\mathcal{M} = \{\text{RMSE}, \text{Vor}, \text{Spec}, \text{Div}\}$ denotes the set of error metrics (see Appendix~\ref{app:metrics} for the metric definitions). The term $\mathcal{E}_m(\mathcal{P})$ represents the error computed for plan $\mathcal{P}$ on metric $m$, and $\mu_m$ is the mean error across all candidates used for normalization. We assign weights $\lambda_{\text{RMSE}}=\lambda_{\text{Vor}}=\lambda_{\text{Spec}}=0.3$ to balance prediction accuracy, physical structure preservation, and spectral fidelity, and $\lambda_{\text{Div}}=0.1$ as a soft physical conservation constraint. These weights were chosen to prioritize pointwise accuracy and physical structure while ensuring approximate mass conservation. This optimization is performed offline, and the resulting compiled tables are reused across all rollouts.

\vspace{0.25em}
\subsection{Training Without Rollout Mismatch}
\label{sec:training}

\noindent The staged rollout enforces a strict information pattern: each site is
predicted from its current-state features plus a bounded neighborhood drawn only
from earlier stages (Section~\ref{sec:executor}). A naive supervised setup can
accidentally leak \emph{future-stage} information during training, creating a
train--test mismatch that amplifies under staged rollout. We avoid this by
training $G_\theta$ with the \emph{same plan-defined inputs} as inference (see
Appendix~\ref{app:plan_aligned_training} for details).

\noindent Concretely, let $\mathbf{v}^\star_{1:N} =
(u^{(\sigma(1))}_{t+1},\ldots,u^{(\sigma(N))}_{t+1})$ be the ground-truth
next-step sequence under order $\sigma$. For each index $i$, we form:
\[
x_i^\star \;=\;
\Big[\;\Phi(U_t,\sigma(i)),\;
\Psi\!\big(\mathbf{v}^\star;\,i\big),\;
\mathrm{pos}(\sigma(i))\;\Big]
\quad \longmapsto \quad \mathbf{v}^\star_i,
\]
and minimize mean-squared error over $i=1{:}N$.

\vspace{0.25em}
\subsection{Scalability and Model Backend}
\label{sec:scalability}

\noindent Inference cost scales with the number of sites $N$, but \sys exposes
simple knobs that keep latency predictable: bounded per-site context and a fixed
number of sequential stages. For very high-resolution grids, we optionally use a
coarse-to-fine execution path that shortens the effective sequence length; and
once the plan fixes the information flow, a lightweight predictor backend is
sufficient. We defer these engineering details to
Appendix~\ref{app:scalability_backend}. In Section~\ref{sec:results}, we show that
this design achieves strong accuracy with low and controllable latency.

\section{Results}
\label{sec:results}
\subsection{Prediction Accuracy Analysis}
\label{sec:accuracy_analysis}

\begin{table*}[htbp]
    \centering
    \caption{Comparison of different models. Best values are \textbf{bolded}, second best are \underline{underlined}. Entries marked with $^{*}$ do not have official hyperparameters reported in the original benchmarks, and we run a short grid search for 5 epochs over canonical architectural choices.}
    \label{tab:metrics_comparison}
    \resizebox{\textwidth}{!}{%
        \setlength{\tabcolsep}{4pt}
        \begin{tabular}{l|cccc|cccc|cccc|cccc|}
        \toprule
        \multicolumn{17}{l}{\textit{CFDBench}} \\
        Dataset & \multicolumn{4}{c|}{Tube} & \multicolumn{4}{c|}{Cylinder} & \multicolumn{4}{c|}{Dam} & \multicolumn{4}{c|}{Cavity} \\
        Model & RMSE & Vor & Spec & Div & RMSE & Vor & Spec & Div & RMSE & Vor & Spec & Div & RMSE & Vor & Spec & Div \\
        \midrule
        FNO & \underline{0.1156} & \underline{0.5330} & 1.2647 & 0.2396 & 0.0915 & 0.1466 & 0.5345 & 0.2046 & \underline{0.0491} & \textbf{0.1705} & 0.5225 & 0.1070 & 0.1911 & 0.0921 & 0.4822 & 0.0650 \\
        U-Net & 0.1461 & 0.7464 & 1.5991 & 0.2402 & 0.0915 & 0.1466 & 0.5289 & 0.2044 & 0.0494 & \underline{0.1742} & \underline{0.5136} & 0.1069 & \textbf{0.1456} & \underline{0.0715} & \underline{0.3608} & \underline{0.0342} \\
        DeepONet & 0.3673 & 0.8676 & \underline{0.8279} & \underline{0.0601} & \underline{0.0709} & \underline{0.1173} & \textbf{0.2193} & \underline{0.0346} & 0.0772 & 0.5063 & \textbf{0.4794} & \underline{0.0347} & 0.2710 & 0.0879 & \textbf{0.3525} & \textbf{0.0301} \\
        CNO$^{*}$ & 0.1873 & 0.8997 & 1.5403 & 0.2306 & 0.0924 & 0.1501 & 0.5583 & 0.2056 & 0.0955 & 0.3624 & 0.6372 & 0.1120 & 0.2385 & 0.0928 & 0.4895 & 0.0587 \\
        UNO$^{*}$ & 0.1301 & 0.7057 & 1.4948 & 0.2272 & 0.1526 & 0.3376 & 0.8219 & 0.2128 & 0.0691 & 0.3458 & 1.1906 & 0.1109 & \underline{0.1476} & \textbf{0.0693} & 0.3883 & 0.0390 \\
        \sys & \textbf{0.1035} & \textbf{0.3326} & \textbf{0.8208} & \textbf{0.0205} & \textbf{0.0258} & \textbf{0.0806} & \underline{0.4325} & \textbf{0.0215} & \textbf{0.0222} & 0.2310 & 0.5303 & \textbf{0.0173} & 0.2949 & 0.2076 & 0.6573 & 0.0919 \\
        \midrule
        \multicolumn{17}{l}{\textit{PDEBench}} \\
        Dataset & \multicolumn{4}{c|}{Diff-React} & \multicolumn{4}{c|}{Rand-M0.1} & \multicolumn{4}{c|}{Rand-M1.0} & \multicolumn{4}{c|}{} \\
        Model & RMSE & Vor & Spec & Div & RMSE & Vor & Spec & Div & RMSE & Vor & Spec & Div & & & & \\[0.4ex]
        \cline{1-13}
        \noalign{\vskip 0.4ex}
        FNO & \underline{0.0850} & \underline{1.0010} & 16.9993 & 2.4304 & \underline{0.1370} & \textbf{0.3684} & \textbf{2.5852} & \underline{5.0081} & \textbf{0.1241} & \textbf{0.7569} & \textbf{4.3602} & \underline{10.5203} & & & & \\
        U-Net & 0.0871 & \textbf{0.7544} & \textbf{1.5762} & \underline{1.8133} & 0.9253 & \underline{1.0838} & 3.9059 & 14.6274 & 1.3545 & \underline{1.2576} & 6.4127 & 20.9948 & & & & \\
        \sys & \textbf{0.0837} & 1.0634 & \underline{1.6317} & \textbf{0.0422} & \textbf{0.0451} & 1.2696 & \underline{2.8513} & \textbf{0.0301} & \underline{0.1351} & 3.4283 & \underline{5.7483} & \textbf{0.0938} & & & & \\
        \midrule
        \multicolumn{17}{l}{\textit{BubbleML}} \\
        Dataset & \multicolumn{4}{c|}{FB-Gravity} & \multicolumn{4}{c|}{PB-Gravity} & \multicolumn{4}{c|}{PB-Subcooled} & \multicolumn{4}{c|}{FB-VelScale} \\
        Model & RMSE & Vor & Spec & Div & RMSE & Vor & Spec & Div & RMSE & Vor & Spec & Div & RMSE & Vor & Spec & Div \\
        \midrule
        FNO & \underline{0.4600} & \underline{1.1274} & \underline{2.5466} & \underline{79.4190} & 0.3304 & \textbf{1.0191} & 2.5540 & \underline{9.7240} & \underline{0.4922} & \textbf{1.0122} & 2.5957 & \underline{8.1798} & \underline{0.9042} & \underline{1.3128} & \textbf{2.0251} & \underline{192.8551} \\
        U-Net & 122.3470$^{*}$ & 512.2977$^{*}$ & 3.3801$^{*}$ & $6.59\times 10^{4}$\,$^{*}$ & \textbf{0.2873} & \underline{1.0467} & \underline{1.9896} & 12.2452 & 0.5030 & \underline{1.0462} & \underline{1.9598} & 12.3974 & 9.8684$^{*}$ & 6.0919$^{*}$ & 3.5103$^{*}$ & $1.35\times 10^{3}$\,$^{*}$ \\
        \sys & \textbf{0.1869} & \textbf{0.9843} & \textbf{2.2159} & \textbf{0.0240} & \underline{0.2898} & 1.0753 & \textbf{1.4567} & \textbf{0.0289} & \textbf{0.3798} & 1.1878 & \textbf{1.5041} & \textbf{0.0371} & \textbf{0.5382} & \textbf{0.9678} & \underline{2.2434} & \textbf{0.0421} \\
        \bottomrule
        \end{tabular}%
    }
\end{table*}

We compare \sys with three representative surrogate families: DeepONet, U-Net,
and FNO, and additionally include CNO/UNO where applicable. All models are
trained for 48 hours on the same hardware with identical train/validation
splits. For \sys, we use the compiler-selected configuration
(Section~\ref{sec:compiler}). For each baseline, we start from the official
reference implementation and tune only standard hyperparameters (learning rate,
batch size, and scheduler) within the same 48-hour budget
(Appendix~\ref{sec:appendix_implementation}). For configurations not covered in the
original benchmarks (e.g., U-Net with FB-Gravity/FB-VelScale or CNO/UNO on
CFDBench), we run a short 5-epoch grid search over canonical architectural
choices and select the best-validation-RMSE setting (see
Appendix~\ref{app:baseline_grid_search} for search spaces and final configurations).
If a method fails to reach a stable training loss within the budget, we report
N/A to reflect that it cannot produce a usable model under this constraint.

We evaluate on eleven datasets from CFDBench, PDEBench, and BubbleML spanning
diverse flow regimes, boundary conditions, and grid resolutions
(Appendix~\ref{app:datasets}). We report RMSE alongside three
physics-oriented diagnostics: vorticity error (rotational structure), spectral
error (frequency content), and divergence error (mass-conservation proxy), all
summarized in \autoref{tab:metrics_comparison}. Beyond one-step evaluation, we
also measure \emph{multi-step} rollout behavior by iteratively feeding
predictions back as inputs and tracking the same diagnostics over time (see
Appendix~\ref{app:multistep_analysis} and \autoref{fig:multistep}).

\sys achieves the best or second-best RMSE on 10/11 datasets, showing that
staged rollout preserves pointwise accuracy relative to strong one-shot
baselines. Its main advantage appears in physical consistency: \sys attains the
lowest divergence error on 10/11 datasets (all except Lid-driven Cavity), often
by a wide margin. This pattern aligns with the staged local rollout mechanism:
each update conditions on an expanding local neighborhood and tends to produce
compatible increments that accumulate fewer global inconsistencies.

The same analysis also surfaces a limitation. On Lid-driven Cavity, \sys is
sub-optimal in RMSE and auxiliary metrics. This dataset is dominated by global
coupling from the elliptic pressure solve in a closed domain: pressure
constraints propagate almost instantaneously across the domain and errors can
recirculate through the boundary. Such dynamics are poorly matched to a purely
local, forward-conditioned rollout that processes sites sequentially, but more
aligned with global operator baselines that directly capture domain-wide
pressure coupling.

\begin{figure}[t]
    \centering \includegraphics[width=\buildfigurewidth{1.0\linewidth}]{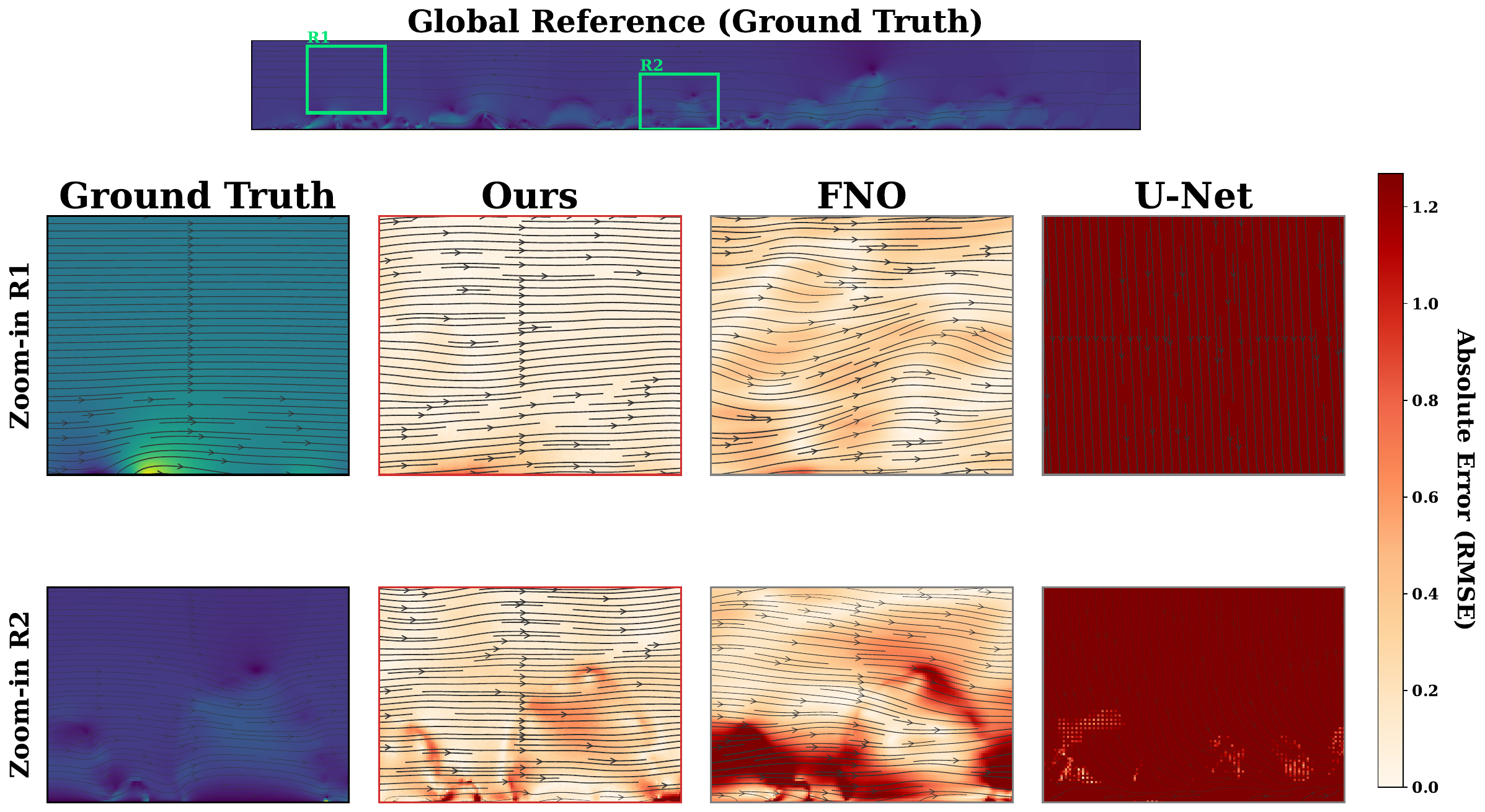}
    \caption{Streamline prediction visualizations for the FB-Gravity dataset.}
    \label{fig:dam_qualitative}
\end{figure}

\autoref{fig:dam_qualitative} zooms into regions of strong shear and vortex
shedding, where small phase or amplitude errors quickly become visually
salient. \sys closely matches the ground-truth streamlines in both regions: in
R1, predicted streamlines are smooth and follow the main flow, while in R2,
vortex cores are captured with little distortion. Errors are sparse and remain
confined to high-vorticity areas, indicating that staged local rollout limits
their spread. By contrast, FNO preserves the bulk flow but blurs
high-frequency structures near interfaces due to global spectral mixing, and
U-Net produces overly uniform streamlines with widespread artifacts, likely
from amplifying local uncertainty across the domain. Overall, this case study
shows that the benefit of \sys is not just lower aggregate error but also
better error geometry: errors are localized, physically interpretable, and tied
to genuine modeling difficulty rather than artifacts of global coupling.

\subsection{Robustness to Input Corruptions} 
\label{sec:robustness} 

We evaluate robustness on CFDBench under Gaussian noise and spatial masking, applied both globally and at edges. We aggregate losses across a range of corruption levels (see Appendix~\ref{sec:appendix_robustness} for details) and normalize them to account for scale. \autoref{fig:robustness_comparison} shows that \sys incurs
the smallest error increase across nearly all perturbations. The advantage is most
pronounced under spatial masking, particularly at boundaries, where global
single-pass predictors suffer substantial error amplification. This behavior
reflects \sys's locality-preserving rollout, which confines corrupted
inputs to bounded neighborhoods and limits cross-domain error propagation.

\begin{figure}[t]
    \centering
    \includegraphics[width=\buildfigurewidth{0.8\linewidth}]{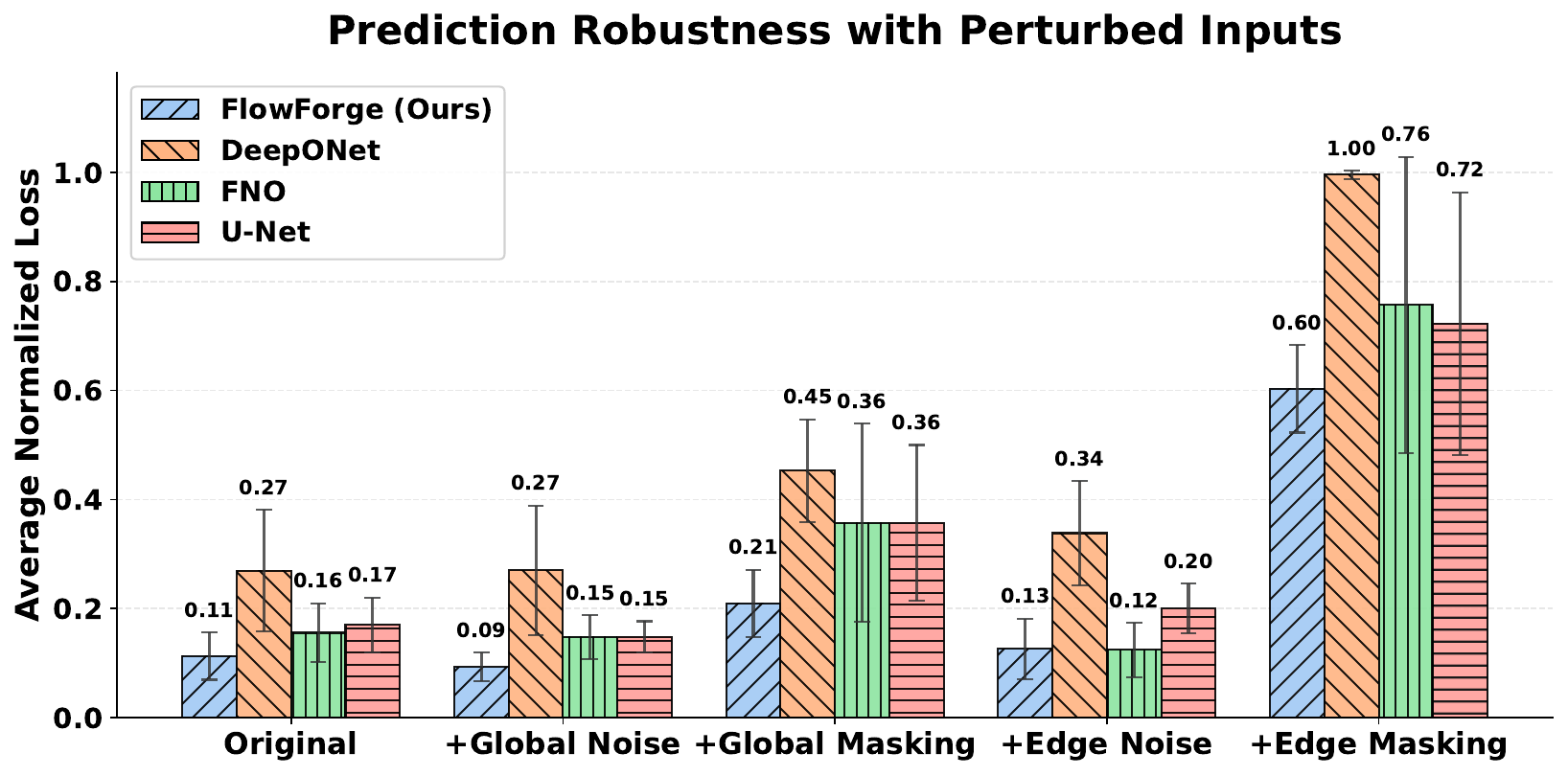}
    \caption{Quantitative evaluation of model robustness under input perturbations. RMSE losses are max-scaled normalized relative to the worst-performing case per dataset. Bar heights indicate the average normalized loss, and error bars represent the standard deviation.} 
    \label{fig:robustness_comparison}
\end{figure}

To make these effects concrete, \autoref{fig:robustness_case_study} shows 
a representative CFDBench example under strong global block masking corruption.
We apply a global mask composed of contiguous spatial blocks and compare
the reconstructed fields. For FNO, U-Net, and DeepONet, the block pattern 
of the mask is clearly imprinted in the predictions: masked regions reappear 
as patchy artifacts or over-smoothed blobs, and the corruption bleeds into 
neighboring, originally clean areas. In contrast, \sys produces a reconstruction 
that is visually close to the ground truth, with almost no visible trace of the mask 
and errors confined to a narrow band surrounding the occluded blocks.

\begin{figure}[t]
    \centering
    \includegraphics[width=\buildfigurewidth{0.8\linewidth}]{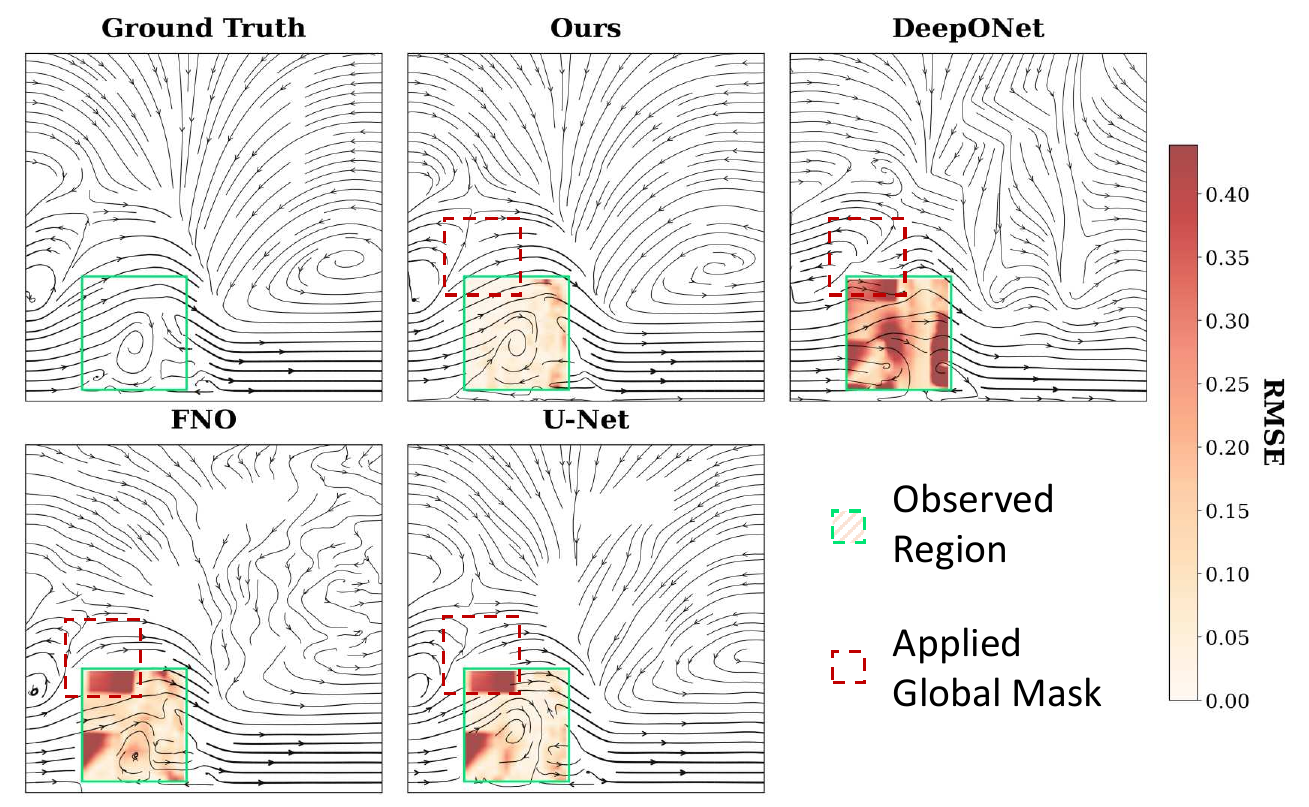}
    \caption{Robustness case study under global block masking corruption.}
    \label{fig:robustness_case_study}
\end{figure}

\subsection{Latency Analysis}
\label{sec:latency_analysis}

\begin{figure}
    \centering
    \includegraphics[width=\buildfigurewidth{0.8\linewidth}]{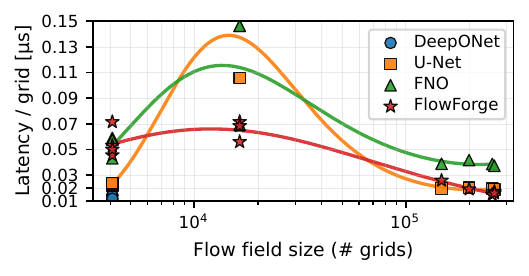}
    \caption{Latency comparison across datasets of increasing resolution. Lower is better.}
    \label{fig:latency}
\end{figure}

We report inference latency for predicting 100 future frames, normalized by the
number of spatial grid points. The normalization
captures how efficiently each method processes a single spatial location as the
problem size grows. \autoref{fig:latency} summarizes results on datasets of
increasing resolution.

\sys exhibits a relatively stable per-point latency across problem sizes.
This behavior is expected given its design: inference is dominated by local,
staged updates whose computational footprint does not grow with the global grid
size. In addition, the \sys compiler explicitly considers latency during
configuration selection, favoring stage schedules and tilings that align well
with the target hardware. As resolution increases, the normalized latency can
decrease because larger grids provide more parallel work per launch and
better amortize fixed overheads, improving hardware utilization.

In contrast, U-Net and FNO show a pronounced increase in normalized latency at
medium grid sizes (around $128\times128$), followed by partial recovery at
larger resolutions. This pattern reflects the cost of global operations---deep
encoder--decoder passes or full-field Fourier transforms---which are inefficient
when intermediate tensors exceed cache capacity but are not yet large enough to
fully amortize kernel overhead. While these methods can become more efficient at
very large scales, the intermediate regime incurs a substantial latency penalty.
DeepONet achieves low per-point latency on small grids but does not scale well to
larger resolutions, as its inference cost grows linearly with the number of
queried spatial locations and lacks reuse across neighboring points. We
therefore restrict DeepONet to coarse grids where it is computationally
competitive.

\subsection{Ablation Study} \label{sec:ablation_study}

\begin{table}[htbp]
    \centering
    \caption{\textbf{Ablation study of key design components on CFDBench-Cylinder.}
    The base configuration is our best setting on this dataset: $K{=}128$, $\sigma{=}$Outward Spiral, and local predictor $G_\theta{=}$MLP.}
    \label{tab:ablation_study} 
    \begin{tabular}{llcc}
        \toprule 
        \textbf{Block} & \textbf{Setting} & \textbf{RMSE} & \textbf{Latency (ms)} \\ 
        \midrule 
        \rowcolor{gray!10}
        Base & \textbf{FlowForge} & 0.0258 & 16.16 \\
        \midrule
        Struct. & One-shot (no rollout) & 0.3457 & 10.47 \\
        \midrule
        $K$ & 4   & 0.0403 & 14.79 \\
           & 16  & 0.0364 & 14.93 \\ 
           & 64  & 0.0291 & 14.88 \\ 
           & 256 & 0.0242 & 16.97 \\
        \midrule 
        $\sigma$ & Raster Scan & 0.0262 & 15.96 \\
                & Random & 0.0265 & 16.64 \\ 
                & Hilbert Curve & 0.0260 & 17.57 \\
        \midrule 
        $G_\theta$ & CNN & 0.3401 & 20.62 \\
                 & FFN & 0.8469 & 21.92 \\
        \bottomrule
    \end{tabular}
    \vspace{1mm}
\end{table}

We ablate core design choices of \sys on CFDBench-Cylinder: (i) whether to perform
a staged within-step rollout or predict all sites in a single parallel pass,
(ii) the rollout granularity $K$, (iii) the traversal order $\sigma$, and (iv) the
local predictor backbone $G_\theta$. \autoref{tab:ablation_study} reports the best
configuration identified by our search (highlighted) together with controlled
variations of each component.

\textbf{Is FlowForge merely enforcing locality? (Local one-shot control).}
One concern is that the observed gains may stem solely from restricting each
prediction to a local neighborhood, rather than from the staged within-step
rollout itself. To disentangle these effects, we introduce a \emph{local
one-shot} baseline that uses the same compiler-selected local context (up to $H$
neighbors per site) and identical input features, but predicts all sites in a
single parallel pass conditioned only on $U_t$ (i.e., without reading from the
partially updated buffer $Y$). Despite preserving locality, this baseline
performs substantially worse (RMSE $0.3457$ vs.\ $0.0258$), indicating that
locality alone without staged rollout is insufficient.

\textbf{Effect of stage granularity ($K$).}
Varying $K$ from $4$ (quasi-parallel) to $256$ (highly sequential) reveals a
consistent accuracy--latency trade-off. Increasing $K$ reduces RMSE by providing
stronger within-step conditioning: more sites can leverage already-updated local
values in $Y$. Latency stays nearly flat up to $K{=}64$ and increases only when
the rollout becomes highly sequential ($K{=}256$), where reduced parallelism
begins to dominate runtime.

\textbf{Effect of traversal order ($\sigma$).}
Orders that preserve locality (Outward Spiral, Raster Scan, Hilbert Curve) yield very similar RMSE,
and even Random is only slightly worse. This indicates that, under our bounded
context selection, performance is governed primarily by \emph{having access to
nearby predecessors}, while the exact geometric path is secondary as long as the
ordering does not systematically deprive updates of local, already-generated
neighbors.

\textbf{Local predictor backbone ($G_\theta$).}
Keeping the staged executor fixed, an MLP is the best-performing local predictor
in both accuracy and latency. CNN/FFN alternatives are substantially worse in
this setting, suggesting that for our compiler-produced sparse predecessor lists
(a fixed set of indexed neighbors rather than a dense image patch), additional
architectural machinery does not translate into better next-step updates and can
incur extra overhead. We therefore use the MLP backbone throughout the paper.

\section{Related Work}
Our work lies at the intersection of deep learning for scientific simulation and sequential modeling. We contrast our paradigm of problem reconstruction with three prevailing approaches in flow field prediction.

\paragraph{Time-series PDE forecasting via learned time-stepping.}
A common approach to PDE forecasting is to learn a one-step time-stepper 
$U_t\!\mapsto\!\widehat{U}_{t+k}$ and roll it out autoregressively, 
making stability, error accumulation, and per-step cost central concerns. 
Recent work spans large-scale operator- and transformer-based autoregressive models
(e.g., DPOT \cite{dpot}, PDEformer \cite{pdeformer}, Unisolver \cite{zhou2025unisolver}), continuous 
space--time representations with symmetry structure
\cite{equivariant-neural-fields}, and analyses that improve
rollout stability via message passing, multiscale stepping, or memory 
\cite{brandstetter2022message,mccabe2023ar_stability,hemmasian2023multiscale,buitrago2025memno,diab2025tno}. 
These methods primarily shape the \emph{temporal semantics} of prediction---how the
one-step map composes across physical time steps---while treating each update as an
atomic operation. In contrast, \sys preserves the same temporal formulation but
intervenes \emph{within} a step: the update is decomposed into ordered, staged
sub-updates so that partial predictions become available as context during the
same physical step.

\paragraph{Backbones for learned time-stepping: neural operators and operator transformers.}
Most learned PDE forecasters implement the one-step map $U_t\!\mapsto\!\widehat{U}_{t+1}$ 
with an \emph{operator-learning} backbone, including classical neural operators 
(e.g., Fourier-/kernel-based operators and operator regression) \cite{fno,deeponet,neuralop_gkn}, their
geometric or multiscale variants (e.g., Geo-FNO, GINO, CNO, factorized FNO, UNO) \cite{geofno,gino,cno,ffno,uno},
and transformer-style operator or PDE-conditioned models
(e.g., GNOT, ONO, HT-Net, Transolver, HAMLET, Poseidon, PDEformer, CORAL) \cite{gnot,ono,htnet,transolver,hamlet,poseidon,pdeformer,coral}. 
These works
are \emph{backbone-centric}: they expand the function class and receptive-field
structure of the predictor while keeping within-step execution fixed, typically
as a synchronous field-wide evaluation. \sys is orthogonal: it leaves the backbone
unchanged but compiles a locality-preserving execution schedule with explicit
intermediate buffers, making causality, conditioning, serial depth, and memory
traffic explicit execution-level design variables.

\paragraph{Benchmarks and evaluation protocols.}
PDE forecasting is typically evaluated using standardized benchmarks such as PDEBench and 
PDEArena, which provide diverse time-dependent PDE datasets and 
reference implementations across surrogate families \cite{pdebench-data,pdearena}. 
CFDBench targets CFD-specific regimes with varying boundary conditions, parameters, 
and geometries \cite{cfdbench}, while BubbleML extends evaluation to multiphase and
multiphysics systems \cite{hassan2023bubbleml}. APEBench focuses on autoregressive 
emulators with rollout-based metrics and controlled protocols \cite{apebench}. 
We follow these conventions to ensure fair comparisons of accuracy, robustness,
 and runtime under realistic rollout settings.

\section{Conclusion}
We presented \sys, a staged local rollout engine for next-step flow field prediction that constrains
information access during inference, enabling accurate local models with predictable cost. Across CFD 
benchmarks, \sys matches or outperforms strong global baselines and is more robust to noise and 
missing data, indicating that rethinking the execution of a prediction step can rival architectural scaling 
in scientific ML and enable more practical surrogates.

\begingroup
\sloppy
\bibliographystyle{unsrtnat}
\bibliography{bibliography}
\endgroup

\appendix
\clearpage
\section{Dataset and Evaluation Metrics}
\label{app:datasets}

To ensure reproducibility and facilitate fair comparison, we provide a comprehensive breakdown of the datasets used in our experiments, spanning three benchmark suites: CFDBench, PDEBench, and BubbleML. Our main evaluation uses eleven datasets (Turb-M0.1 is excluded from the main results). In this section, we detail the governing equations and physical parameters for each dataset.

\subsection{CFDBench}
\label{sub:cfdbench_details}

We utilize four classical 2D fluid dynamics scenarios from CFDBench \cite{cfdbench}. These flows are governed by the Navier-Stokes equations for an incompressible Newtonian fluid. 

\textbf{Governing Equations.} 
Let $\mathbf{u} = (u, v)^\top$ denote the velocity field, $p$ the pressure, $\rho$ the fluid density, and $\mu$ the dynamic viscosity. The dynamics follow:
\begin{equation*}
\label{eq:cfdbench_ns}
\begin{cases}
    \nabla \cdot \mathbf{u} = 0 \\
    \displaystyle \frac{\partial \mathbf{u}}{\partial t} + (\mathbf{u} \cdot \nabla) \mathbf{u} = -\frac{1}{\rho} \nabla p + \frac{\mu}{\rho} \nabla^2 \mathbf{u}
\end{cases}
\end{equation*}
The boundary conditions are defined such that $\mathbf{u}$ is constant on the domain boundaries $\partial \mathcal{D}$. The four specific configurations are:

\begin{itemize}
    \item \texttt{Cavity} (Lid-driven Cavity): Flow in a square container driven by a moving top lid. This dataset poses a challenge due to the mathematical singularities at the corners where the moving lid meets stationary walls, requiring the model to resolve sharp gradients and recirculating vortices.
    \item \texttt{Tube} (Circular Tube Flow): Represents the development of a laminar boundary layer. Physically, viscous resistance near the wall creates a parabolic velocity profile. This task tests the model's ability to capture boundary layer theory and liquid-gas interfaces in a confined geometry.
    \item \texttt{Dam} (Dam Break): Models the transient release of fluid over an obstacle using the Volume of Fluid (VOF) method. The flow regime transitions dynamically from viscosity-dominated (creeping flow) to inertia-dominated (jet formation), testing generalization across different physical force balances.
    \item \texttt{Cylinder} (Flow Around a Cylinder): A classic aerodynamic problem where boundary layer separation creates a Karman vortex street. This is the primary benchmark for modeling periodic oscillatory behaviors induced by rigid obstacles.
\end{itemize}

\subsection{PDEBench}
\label{sub:pdebench_details}

The selected datasets from PDEBench \cite{pdebench-data} introduce compressibility and non-linear chemical dynamics, significantly expanding the physical scope beyond standard CFD tasks.

\textbf{Compressible Navier-Stokes (\texttt{Rand-M0.1}, \texttt{Rand-M1.0}, \texttt{Turb-M0.1}).}
These datasets evolve density $\rho$, velocity $\mathbf{v}$, and pressure $p$ coupled with internal energy, governed by the full conservation laws:
\begin{equation*}
\label{eq:compressible_ns}
    \partial_t \rho + \nabla \cdot (\rho \mathbf{v}) = 0, \quad
    \rho(\partial_t \mathbf{v} + \mathbf{v} \cdot \nabla \mathbf{v}) = -\nabla p + \nabla \cdot \sigma'
\end{equation*}
To ensure precise reproducibility, we specify the exact subsets used from the PDEBench repository. The datasets differ in both compressibility (Mach number $M$) and viscosity parameters (shear $\eta$ and bulk $\zeta$):
\begin{itemize}
    \item \texttt{Rand-M0.1}: Corresponds to file \texttt{2D\_Rand\_M0.1\_Eta0.01\_Zeta0.01}. This represents a subsonic regime ($M < 0.1$) with low viscosity ($\eta=\zeta=10^{-2}$), serving as a baseline bridging incompressible and compressible physics.
    \item \texttt{Rand-M1.0}: Corresponds to file \texttt{2D\_Rand\_M1.0\_Eta0.1\_Zeta0.1}. This represents a transonic regime ($M \le 1.0$) with shock waves. Note that this dataset uses higher viscosity ($\eta=\zeta=10^{-1}$) than M0.1. The model must capture sharp discontinuities where gradients become theoretically infinite.
    
\end{itemize}

\textbf{2D Diffusion-Reaction (\texttt{Diff-React}).}
This dataset models biological pattern formation (activator $u$, inhibitor $v$) via the FitzHugh-Nagumo equations:
\begin{equation*}
    \partial_t u = D_u \Delta u + R_u(u, v), \quad \partial_t v = D_v \Delta v + R_v(u, v)
\end{equation*}
where non-linear reaction terms $R$ and differential diffusion rates ($D_v > D_u$) lead to complex Turing patterns, testing the model's capacity to maintain spectral fidelity over long rollouts.

\subsection{BubbleML}
\label{sub:bubbleml_details}

BubbleML \cite{hassan2023bubbleml} represents the most physically complex benchmark in our evaluation, involving multiphase flows with heat transfer and phase change. The simulations use FC-72 (perfluorohexane) as the working fluid.

\textbf{Governing Equations.} 
The system solves the non-dimensionalized incompressible Navier-Stokes equations coupled with the energy equation. The liquid ($l$) and vapor ($v$) phases are tracked using a level-set function $\phi$, where the zero level set $\phi=0$ represents the interface $\Gamma$. The dynamics evolve according to:
\begin{equation*}
\label{eq:bubbleml_governing}
\begin{aligned}
    \frac{\partial \mathbf{u}}{\partial t} + \mathbf{u} \cdot \nabla \mathbf{u} &= -\frac{1}{\rho'} \nabla P + \nabla \cdot \left[ \frac{\mu'}{\rho' \text{Re}} \nabla \mathbf{u} \right] + \frac{\mathbf{g}}{\text{Fr}^2} + \mathbf{S}^\Gamma_{\mathbf{u}} + \mathbf{S}^\Gamma_{P} \\
    \frac{\partial T}{\partial t} + \mathbf{u} \cdot \nabla T &= \nabla \cdot \left[ \frac{\alpha'}{\text{Re} \text{Pr}} \nabla T \right] + S^\Gamma_T
\end{aligned}
\end{equation*}
Here, $\mathbf{u}$ is velocity, $P$ is pressure, and $T$ is temperature. The system dynamics are controlled by key dimensionless numbers: Reynolds number ($\text{Re}$), Prandtl number ($\text{Pr}$), and Froude number ($\text{Fr}$). The terms $\rho', \mu', \alpha'$ denote density, viscosity, and thermal diffusivity scaled by their liquid-phase values.

Crucially, the terms $\mathbf{S}^\Gamma$ represent singular source terms at the interface $\Gamma$:
\begin{itemize}
    \item $\mathbf{S}^\Gamma_{\mathbf{u}}$ accounts for surface tension forces, modeled via the Weber number ($\text{We}$).
    \item $\mathbf{S}^\Gamma_{P}$ accounts for the pressure jump induced by phase change.
    \item $S^\Gamma_T$ represents the latent heat source/sink during evaporation or condensation.
\end{itemize}

\textbf{Dataset Configurations.} We utilize four datasets representing two boiling regimes:
\begin{itemize}
    \item \textbf{Pool Boiling (\texttt{PB-Gravity}, \texttt{PB-Subcooled}):} 
    Simulates fluid confined in a tank with a bottom heater (resembling nuclear waste cooling). The \texttt{PB-Subcooled} variant introduces a temperature deficit in the bulk liquid to suppress bubble growth.
    \item \textbf{Flow Boiling (\texttt{FB-Gravity}, \texttt{FB-VelScale}):} 
    Models fluid pumped through a channel containing a heater (resembling data center GPU liquid cooling). These datasets feature extreme aspect ratios (e.g., $1600 \times 160$) and complex interactions between bulk momentum and buoyancy-driven bubbles.
\end{itemize}

\begin{table*}[!ht]
\centering
\caption{Detailed specifications of the datasets. \textbf{Size} denotes the total storage size of the dataset in Gigabytes (GB). The "Physics" column references the specific governing equations defined in the text. Note: Turb-M0.1 is excluded from the main evaluation results.}
\label{tab:dataset_specs}
\vspace{5pt}
\resizebox{\textwidth}{!}{%
\begin{tabular}{l l c l l c}
\toprule
\textbf{Suite} & \textbf{Dataset} & \textbf{Resolution} & \textbf{Physics / PDE} & \textbf{Boundary Conditions} & \textbf{Size (GB)} \\
\midrule
\multirow{4}{*}{\textbf{CFDBench}} 
 & \texttt{Cavity} & $64\times64$ & Incompressible NS (Eq.~\ref{eq:cfdbench_ns}) & Wall & \textbf{0.8} \\
 & \texttt{Tube} & $64\times64$ & Incompressible NS (Eq.~\ref{eq:cfdbench_ns}) & Inflow/Outflow & \textbf{0.3} \\
 & \texttt{Dam} & $64\times64$ & Multiphase (VOF) & Wall & \textbf{1.3} \\
 & \texttt{Cylinder} & $64\times64$ & Incompressible NS (Eq.~\ref{eq:cfdbench_ns}) & Inflow/Outflow & \textbf{12.5} \\
\midrule
\multirow{4}{*}{\textbf{PDEBench}} 
 & \texttt{Diff-React} & $128\times128$ & Reaction-Diffusion & Periodic & \textbf{12.3} \\
 & \texttt{Rand-M0.1} & $128\times128$ & Compressible NS ($M < 0.1$) & Periodic & \textbf{51.2} \\
 & \texttt{Rand-M1.0} & $128\times128$ & Compressible NS ($M \le 1.0$) & Periodic & \textbf{51.2} \\
 & \texttt{Turb-M0.1} & $512\times512$ & Compressible Turbulence & Periodic & \textbf{82.0} \\
\midrule
\multirow{4}{*}{\textbf{BubbleML}} 
 & \texttt{FB-Gravity} & $1600\times160$ & Flow Boiling & Wall/Outflow & \textbf{16.1} \\
 & \texttt{PB-Gravity} & $512\times512$ & Pool Boiling & Wall & \textbf{24.2} \\
 & \texttt{PB-Subcooled} & $384\times384$ & Pool Boiling & Wall & \textbf{15.4} \\
 & \texttt{FB-VelScale} & $1244\times160$ & Flow Boiling & Wall/Outflow & \textbf{15.7} \\
\bottomrule
\end{tabular}
}
\end{table*}

\subsection{Evaluation Metrics.} 
\label{app:metrics}

To provide a comprehensive assessment of physical fidelity, we employ four distinct metrics targeting prediction accuracy, physical structure preservation, frequency consistency, and physical conservation. Let $U \in \mathbb{R}^{N \times C \times H \times W}$ denote the ground truth flow field batch and $\hat{U}$ denote the prediction.

\begin{description}
    \item[Root Mean Squared Error (RMSE).] 
    We measure the absolute pixel-level deviation aggregated globally over the entire dataset. This formulation ensures that the metric is insensitive to batch size variations and penalizes large outliers effectively:
    \begin{equation*}
        \mathrm{RMSE} = \sqrt{\frac{1}{N \cdot C \cdot H \cdot W} \sum_{i,c,x,y} \left( U - \hat{U} \right)^2 }
    \end{equation*}

    \item[Relative Vorticity Error (Vor).] 
    To evaluate the capture of rotational dynamics and small-scale structures, we compute the relative Frobenius norm error of the vorticity field $\omega = \nabla \times U$. Using relative error accounts for varying flow magnitudes across different samples:
    \begin{equation*}
        \mathrm{VorError} = \frac{\| \omega - \hat{\omega} \|_F}{\| \omega \|_F} = \frac{\sqrt{\sum (\omega - \hat{\omega})^2}}{\sqrt{\sum \omega^2}}
    \end{equation*}
    where $\|\cdot\|_F$ denotes the Frobenius norm over the spatial dimensions and batch.

    \item[Spectral Error (Spec).] 
    Standard metrics often ignore high-frequency details due to the spectral bias of neural networks. We compute the error in the frequency domain using the Log Spectral Distance to balance low- and high-frequency contributions:
    \begin{equation*}
        \begin{aligned}
        \mathrm{SpecError}
        &= \frac{1}{N} \sum_{i=1}^N
        \Bigl\|
        \log\!\left(|\mathcal{F}(U_i)| + \epsilon\right)
        \\
        &\qquad
        - \log\!\left(|\mathcal{F}(\hat{U}_i)| + \epsilon\right)
        \Bigr\|_2
        \end{aligned}
    \end{equation*}
    where $\mathcal{F}$ denotes the 2D Fast Fourier Transform (FFT) applied spatially, and $\epsilon=10^{-8}$ ensures numerical stability.

    \item[Divergence Error (Div).] 
    For incompressible flows, mass conservation dictates that the divergence $\nabla \cdot U$ should be zero. We measure the root mean squared deviation of the predicted divergence:
    \begin{equation*}
        \mathrm{DivError} = \sqrt{ \frac{1}{N \cdot H \cdot W} \sum_{i,x,y} \left( (\nabla \cdot U) - (\nabla \cdot \hat{U}) \right)^2 }
    \end{equation*}
    Here, $\nabla \cdot U$ is computed using second-order central finite differences for interior points.
\end{description}

\clearpage
\section{Implementation Details}
\label{sec:appendix_implementation}

\subsection{Codebase}
\sys is implemented in Python with approximately 10k lines of code. We will release the full code upon publication.

\subsection{Architectures and Backend}
\label{app:impl_details}

\paragraph{\sys predictor and backend model.}
Once the plan fixes what information is visible (bounded local context from
earlier stages plus static site features), we instantiate $G_\theta$ as a
compact, fixed predictor across all benchmarks: a 4-layer MLP with
$\approx 6\times 10^{4}$ parameters. Given local features, the model computes
\[
\big[\Phi(U_t,\sigma(i)),\;\Psi(Y;i),\;\mathrm{pos}(\sigma(i))\big]
\;\mapsto\; \widehat{y}_i \in \mathbb{R}^{c}.
\]
This emphasizes the core design message: reliability and controllability come
from the staged information flow enforced by the plan and executor, rather than
from global attention or deep encoder--decoder stacks.

\paragraph{Baselines.}
For U-Net, FNO, and DeepONet, we use the reference implementations and default
architectures provided by the corresponding benchmarks/datasets.

\subsection{Hardware and Training Protocol}
All experiments were run on a single NVIDIA RTX~3090 GPU (24\,GB).
Unless otherwise specified, we train all models with Adam (learning rate $10^{-3}$),
MSE loss, and a StepLR schedule for up to 150 epochs, with a wall-clock cap of 48 hours.
Batch size is chosen per method/dataset to maximize GPU utilization without exceeding VRAM.

\subsection{Plan-aligned Training Inputs}
\label{app:plan_aligned_training}

During inference, the executor forms each input from the current-state features
and a bounded neighborhood gathered from the partially rewritten buffer, with the
guarantee that all gathered entries come from earlier stages
(Section~\ref{sec:executor}). To avoid training--inference mismatch, we reuse the
same plan-defined neighborhood index lists when constructing training inputs.
When those indices would read from the inference-time buffer, we instead read
from the ground-truth next-step sequence $\mathbf{v}^\star_{1:N}$, but only at
indices that are valid under the stage order. This preserves the executor's
information constraints while avoiding the instability of feeding intermediate
model predictions during training.

\subsection{Inference Efficiency and Scalability}
\label{app:scalability_backend}

\noindent \textbf{CUDA Graphs.}
We execute \sys with CUDA Graphs. We capture the full inference workflow into a
static CUDA graph and replay it for each input, reducing CPU launch overhead to
a single replay call and mitigating GPU idle time caused by frequent kernel
launches. For fair comparison, all baselines are also optimized with CUDA
Graphs using the same approach.

\medskip
\noindent \textbf{Coarse-to-fine execution.} For very high-resolution grids, we
optionally reduce the effective sequence length by running \sys on a downscaled
field and then upsampling the prediction back to the native grid. Let
$\mathcal{D}_\rho$ downscale by factor $\rho$ per axis and $\mathcal{R}_\rho$
upsample back. We compute:
\[
\widehat{U}^{\,\mathrm{coarse}}_{t+1} \;=\; \sys\big(\mathcal{D}_\rho(U_t)\big),
\qquad
\widehat{U}_{t+1} \;=\; \mathcal{R}_\rho\!\big(\widehat{U}^{\,\mathrm{coarse}}_{t+1}\big).
\]

\subsection{Ordering Schemes for Tokenization}
\label{app:ordering}

Let $\Omega$ denote the set of grid coordinates (2D/3D) with $N=|\Omega|$.
An ordering is a bijection $\sigma:\{1,\dots,N\}\to\Omega$.

\paragraph{Outward Spiral}
Let $c$ be the integer-rounded grid center. Define $\|p-c\|_\infty=\max_j |p_j-c_j|$.
Outward Spiral ordering enumerates sites by increasing Chebyshev radius:
\[
\sigma \;=\; \bigcup_{r=0}^{r_{\max}} \{\, p\in\Omega \mid \|p-c\|_\infty = r \,\},
\]
where ties within the same radius are broken by a fixed deterministic scan over the
bounding box of the shell (and duplicates are skipped).

\paragraph{Raster Scan}
Lexicographic (row-major) order:
2D: $(x,y)$ with $y$ outer and $x$ inner.
3D: $(x,y,z)$ with $z$ outer, then $y$, then $x$.

\paragraph{Random}
A uniform random permutation of $\Omega$, generated once with a fixed RNG seed and reused.

\paragraph{Hilbert Curve}
Embed the grid into the smallest power-of-two cube/box with side
$M=2^{\lceil \log_2(\max(X,Y(,Z)))\rceil}$.
Assign each $p\in\Omega$ a Hilbert Curve index $h(p)\in[0,M^d)$ and sort by $h(p)$; embedded
out-of-bounds points are ignored.

Orders that preserve spatial locality (e.g., Outward Spiral and Raster Scan) tend to behave similarly,
whereas orders that fail to expose nearby predecessors early under a fixed context budget can degrade.

\clearpage
\section{Baseline Hyperparameter Search}
\label{app:baseline_grid_search}

For baseline settings that are not reported in the original benchmark suites,
we run a compact architecture sweep around the canonical public
implementations while keeping the optimizer, scheduler, and overall 48-hour
training budget unchanged. Each candidate first receives a 5-epoch proxy run,
and we promote the setting with the best validation RMSE to the full training
run. We apply this procedure to CNO and UNO on CFDBench and to U-Net on the
two BubbleML cases without benchmark-provided settings
(\texttt{FB-Gravity} and \texttt{FB-VelScale}).

We retain the implementation names \texttt{kernel\_size} for CNO and
\texttt{uno\_kernel\_size} for UNO; both denote the local operator kernel
width used inside the corresponding architecture.

\subsection{CNO on CFDBench}

\paragraph{Parameters.}
We search the three CNO knobs that most directly shape capacity and local
aggregation. \texttt{hidden\_dim} sets the feature width after the lifting
layer and therefore the channel capacity carried by every operator block,
making it the main width parameter. \texttt{depth} is the number of stacked
operator blocks, so it determines how much iterative feature refinement the
model can perform under a fixed training budget. \texttt{kernel\_size}
controls the local receptive field of the depthwise operator in each block and
is therefore the key parameter for how much near-field spatial context each
update can absorb. We do not tune padding separately, because the
implementation deterministically sets
\texttt{cno\_padding} = (\texttt{kernel\_size} - 1) / 2 to preserve spatial
resolution.

\paragraph{Search spaces and results.}
We sweep:
\begin{align*}
    \texttt{hidden\_dim} &\in \{64, 96, 128\}, \\
    \texttt{depth} &\in \{4, 6, 8\}, \\
    \texttt{kernel\_size} &\in \{3, 5\}.
\end{align*}
For each CFDBench dataset, we choose the setting with the lowest validation
RMSE among the 5-epoch proxy runs. The final selected configurations are:
\begin{center}
\begin{tabular}{lccc}
\toprule
Dataset & depth & hidden\_dim & kernel\_size \\
\midrule
Tube     & 4 & 64  & 3 \\
Cylinder & 6 & 96  & 5 \\
Dam      & 8 & 128 & 5 \\
Cavity   & 8 & 128 & 5 \\
\bottomrule
\end{tabular}
\end{center}

\subsection{UNO on CFDBench}

\paragraph{Parameters.}
We search the UNO parameters that most strongly shape its multiscale
hierarchy. \texttt{uno\_base\_dim} sets the channel width at the finest scale
and therefore the width of the full encoder--decoder stack.
\texttt{uno\_levels} controls how many resolutions the model traverses, so it
changes both receptive-field growth and total multiscale capacity.
\texttt{uno\_kernel\_size} determines the local receptive field inside each
operator block and therefore how aggressively each stage aggregates nearby
context. We keep \texttt{uno\_depth\_per\_level}=2 and
\texttt{uno\_bottleneck\_depth}=4 fixed. These parameters regulate how many
blocks are repeated within each scale and at the bottleneck, so holding them
constant keeps the sweep focused on width, hierarchy depth, and local context
size instead of exploding the search space.

\paragraph{Search spaces and results.}
We sweep:
\begin{align*}
    \texttt{uno\_base\_dim} &\in \{48, 64, 96\}, \\
    \texttt{uno\_levels} &\in \{3, 4\}, \\
    \texttt{uno\_kernel\_size} &\in \{3, 5\},
\end{align*}
while keeping the following structural parameters fixed:
\begin{align*}
    \texttt{uno\_depth\_per\_level} &= 2, \\
    \texttt{uno\_bottleneck\_depth} &= 4.
\end{align*}
For each CFDBench dataset, we again select the setting with the lowest
validation RMSE from the 5-epoch proxy runs. The resulting chosen
configurations are:
\begin{center}
\begin{tabular}{lccc}
\toprule
Dataset & base\_dim & levels & kernel\_size \\
\midrule
Tube     & 48 & 3 & 3 \\
Cylinder & 64 & 4 & 5 \\
Dam      & 96 & 4 & 5 \\
Cavity   & 96 & 4 & 5 \\
\bottomrule
\end{tabular}
\end{center}

\subsection{U-Net on BubbleML}

\paragraph{Parameters.}
For BubbleML we keep the U-Net topology used in our experiments fixed: four
encoder blocks separated by $2 \times 2$ max-pooling layers, a bottleneck
block, and four mirrored decoder blocks. Each decoder stage first upsamples
with a $2 \times 2$ transposed convolution, concatenates the corresponding
encoder skip feature, and then applies the decoder block, followed by a final
$1 \times 1$ output projection. Under this implementation,
\texttt{init\_features} is the only exposed width knob; it sets the channel
count of the first encoder block. It is therefore the main
architecture parameter for trading off model capacity against memory footprint,
while the overall encoder--decoder topology remains fixed.

\paragraph{Search spaces and results.}
We search only:
\begin{align*}
    \texttt{init\_features} \in \{32, 64, 128\}.
\end{align*}
Exhaustive BubbleML sweeps are prohibitively expensive, so each candidate
receives a 5-epoch proxy run on \texttt{Diff-React}, and we choose the width
with the lowest validation RMSE. This selects
\texttt{init\_features}=64, which we then transfer unchanged to
\texttt{FB-Gravity} and \texttt{FB-VelScale}.

\clearpage
\section{Additional Results and Analysis}

\subsection{Multi-step Autoregressive Rollout Analysis}
\label{app:multistep_analysis}

We evaluate long-horizon autoregressive rollouts by iteratively feeding model
predictions back as inputs and tracking error metrics over time
(\autoref{fig:multistep}) on CFDBench-Cylinder. While all methods exhibit
growing RMSE with rollout length, FlowForge remains comparable to FNO and more
stable than U-Net, indicating that staged rollout does not introduce additional
numerical instability. The most pronounced difference appears in physical
consistency: FlowForge maintains near-zero divergence error throughout the
rollout, whereas FNO and U-Net exhibit steadily increasing divergence,
reflecting accumulation of global inconsistency. FlowForge also preserves
vorticity and spectral content at levels comparable to or better than
baselines, despite relying on strictly local updates. Overall, these results
show that the staged local rollout mechanism yields stable long-horizon
behavior and effectively limits constraint drift under autoregressive use.

\subsection{Analysis of Effective Physical Dependency}
\label{app:dependency}

In the ablation study (Section~\ref{sec:ablation_study}), we observed that \sys
exhibits remarkable performance robustness across various spatial traversal
orders $\sigma$ (e.g., Outward Spiral, Raster Scan, Hilbert Curve, and Random).
Intuitively, one might expect that locality-preserving orders (like Outward
Spiral) would significantly outperform Random ordering. To understand why this
performance gap is minimal, we analyze the \textit{effective physical
dependency} provided by each traversal strategy.

Based on the causal inference structure analyzed in
Section~\ref{sec:rollout_structure}, we posit that the prediction accuracy at a
target location $x_t$ depends critically on the availability of spatially
proximate context from the set of previously predicted locations
$\mathcal{P}_t = \{x_{\sigma(1)}, \dots, x_{\sigma(t-1)}\}$. If a target
location is far from any known value, the predictor lacks immediate physical
cues.

We quantify this dependency using the \textbf{Nearest Preceding Neighbor
Distance}, defined for the $t$-th step in the sequence as:
\begin{equation*}
    d_t = \min_{j < t} \| x_{\sigma(t)} - x_{\sigma(j)} \|_2.
\end{equation*}
A smaller $d_t$ implies stronger local conditioning. We simulate the traversal
processes on a $16 \times 16$ grid and record the evolution of $d_t$; note that
in \sys the plan-defined context for a target location is an \emph{irregular
list} of up to $H$ earlier neighbor sites (within $R$ rings/shells), rather
than a dense $h{\times}w$ image patch.

\autoref{fig:dependency_dist} illustrates the moving average of $d_t$ for
Outward Spiral, Raster Scan, Hilbert Curve, and Random orders. Our analysis
reveals two distinct phases:

\begin{enumerate}
    \item \textbf{Warm-up Phase ($t < 30$):} In the initial steps, locality-preserving orders (Outward Spiral, Raster Scan, Hilbert Curve) maintain a consistent $d_t = 1$ (adjacent pixels), whereas Random ordering exhibits high variance and larger average distances ($d_t > 1$), as early points are scattered sparsely across the grid.

    \item \textbf{Stable Phase ($t \ge 30$):} As the grid becomes populated, the probability of a randomly selected target location falling within the immediate neighborhood of an existing point increases distinctly. Consequently, the $d_t$ for Random ordering rapidly decays and converges to the same baseline ($d_t \approx 1$), providing similar context vectors as the locality-preserving orders.
\end{enumerate}

Since the local predictor is trained to map from a set of context vectors to
the target value, its performance is dominated by the availability of close
neighbors rather than the global geometry of the sequence. Because the
``warm-up'' phase represents a negligible fraction of the total rollout steps
(30 out of 256 pixels), the aggregate loss is dominated by the stable phase,
where all traversal orders provide statistically similar physical context.
Thus, \sys is not strictly bound to a specific traversal order.
\begin{figure}[!b]
    \centering
    \includegraphics[width=\buildfigurewidth{1.0\linewidth}]{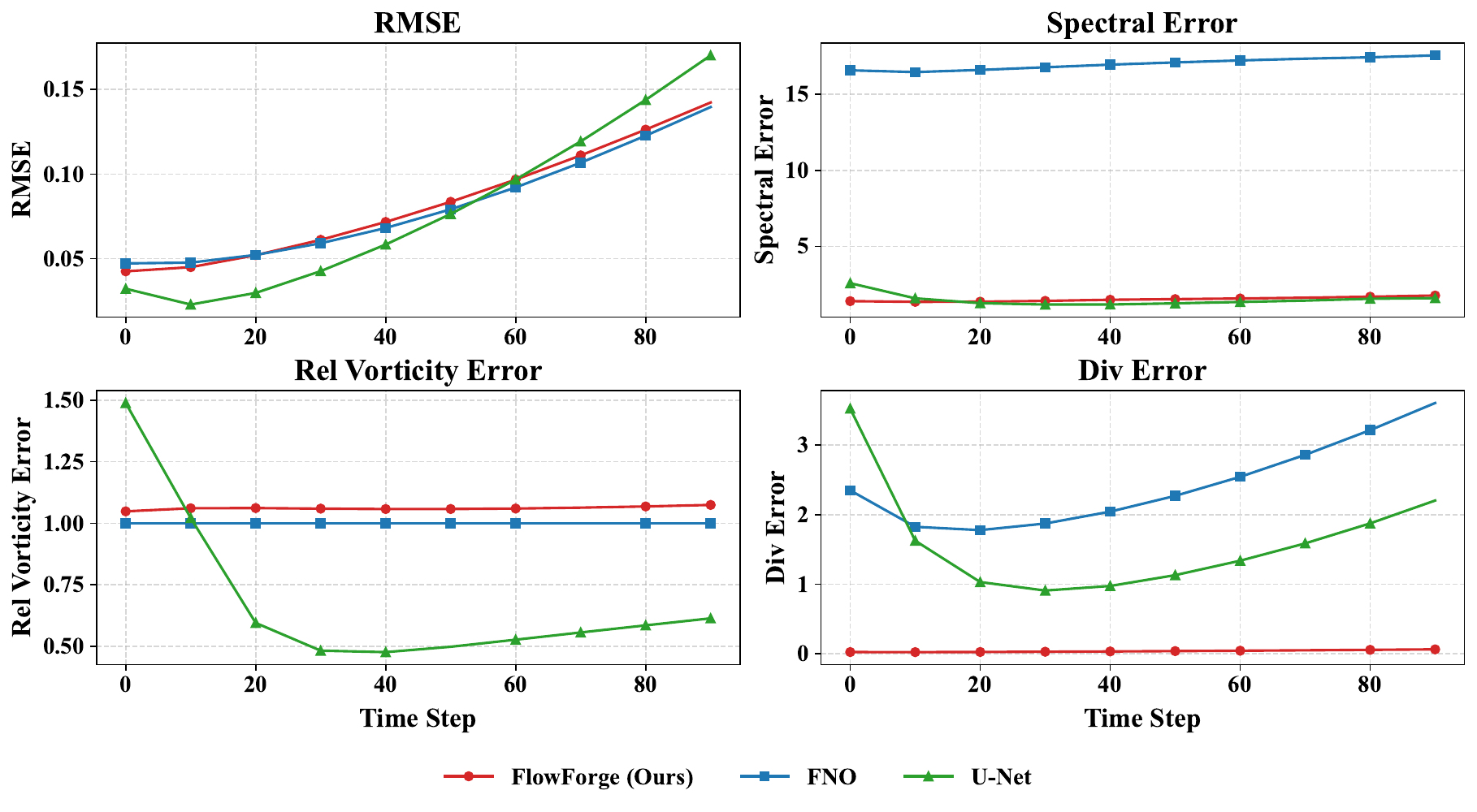}
    \caption{\textbf{Multi-step autoregressive rollout on CFDBench-Cylinder.}
    We report RMSE, spectral error, relative vorticity error, and divergence error
    over rollout time steps for FlowForge (ours), FNO, and U-Net.}
    \label{fig:multistep}
\end{figure}

\begin{figure}[!b]
    \centering
    \includegraphics[width=\buildfigurewidth{0.8\linewidth}]{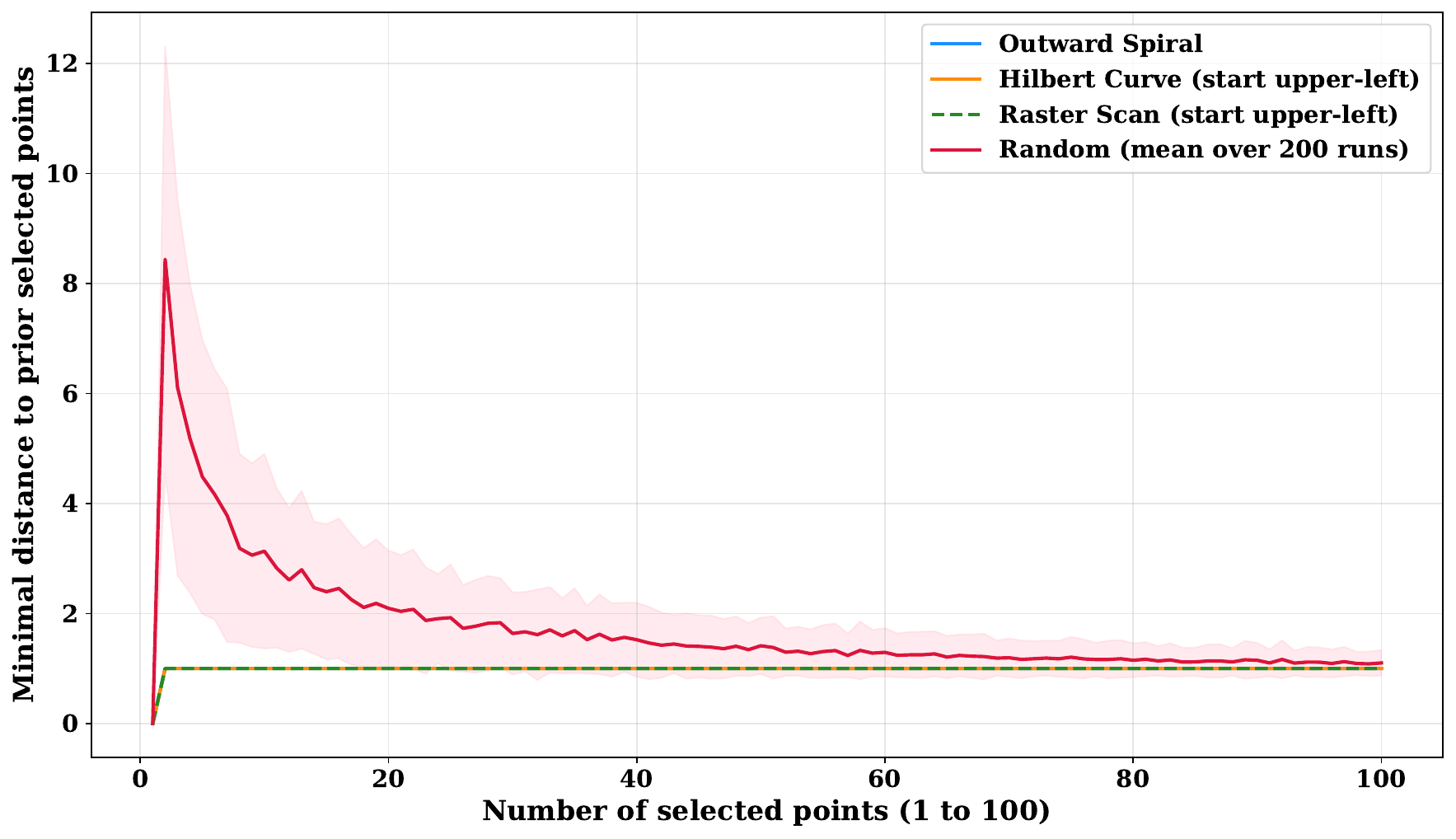}
    \caption{\textbf{Evolution of Nearest Preceding Neighbor Distance ($d_t$).} Comparison of different traversal orders on a $16 \times 16$ grid. While Random ordering starts with higher separation distances, it quickly converges to the immediate-neighbor regime ($d_t \approx 1$) after roughly 30 steps, providing similar local conditioning context to structured orders like Outward Spiral or Hilbert Curve.}
    \label{fig:dependency_dist}
\end{figure}

\FloatBarrier
\subsection{Complete Robustness Test Results}
\label{sec:appendix_robustness}

This section presents comprehensive robustness evaluations across all four CFDBench datasets, complementing the findings in Section~\ref{sec:robustness}.
We report the average loss computed over distinct severity levels: noise standard deviation $\sigma \in \{0.01, 0.05, 0.1\}$, edge width $w \in \{1, 5, 10\}$, and masking rate $k \in \{1\%, 5\%, 10\%\}$. 
As shown in Figure~\ref{fig:appendix_robustness_grid}, \sys achieves the
lowest or second-lowest mean RMSE under all four perturbation settings across
every dataset in CFDBench. These comprehensive results reinforce the analysis
in Section~\ref{sec:robustness} and indicate that \sys is consistently more
resilient to input perturbations than the baselines.

\begingroup
    \centering
    \begin{minipage}[t]{0.48\linewidth}
        \centering
        \textbf{CFDBench-Cylinder}\par\smallskip
        \includegraphics[width=\linewidth]{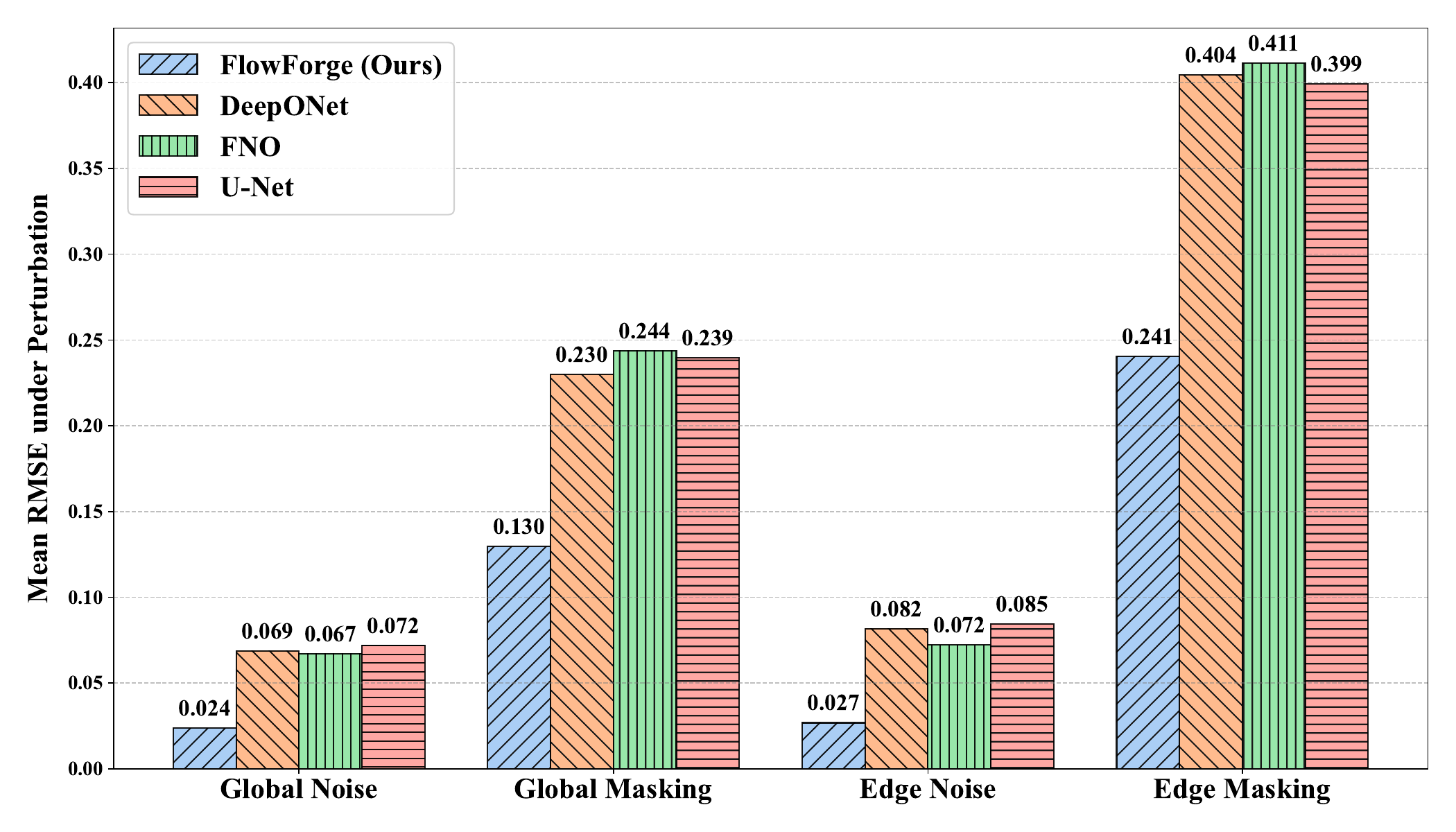}
    \end{minipage}\hfill
    \begin{minipage}[t]{0.48\linewidth}
        \centering
        \textbf{CFDBench-Cavity}\par\smallskip
        \includegraphics[width=\linewidth]{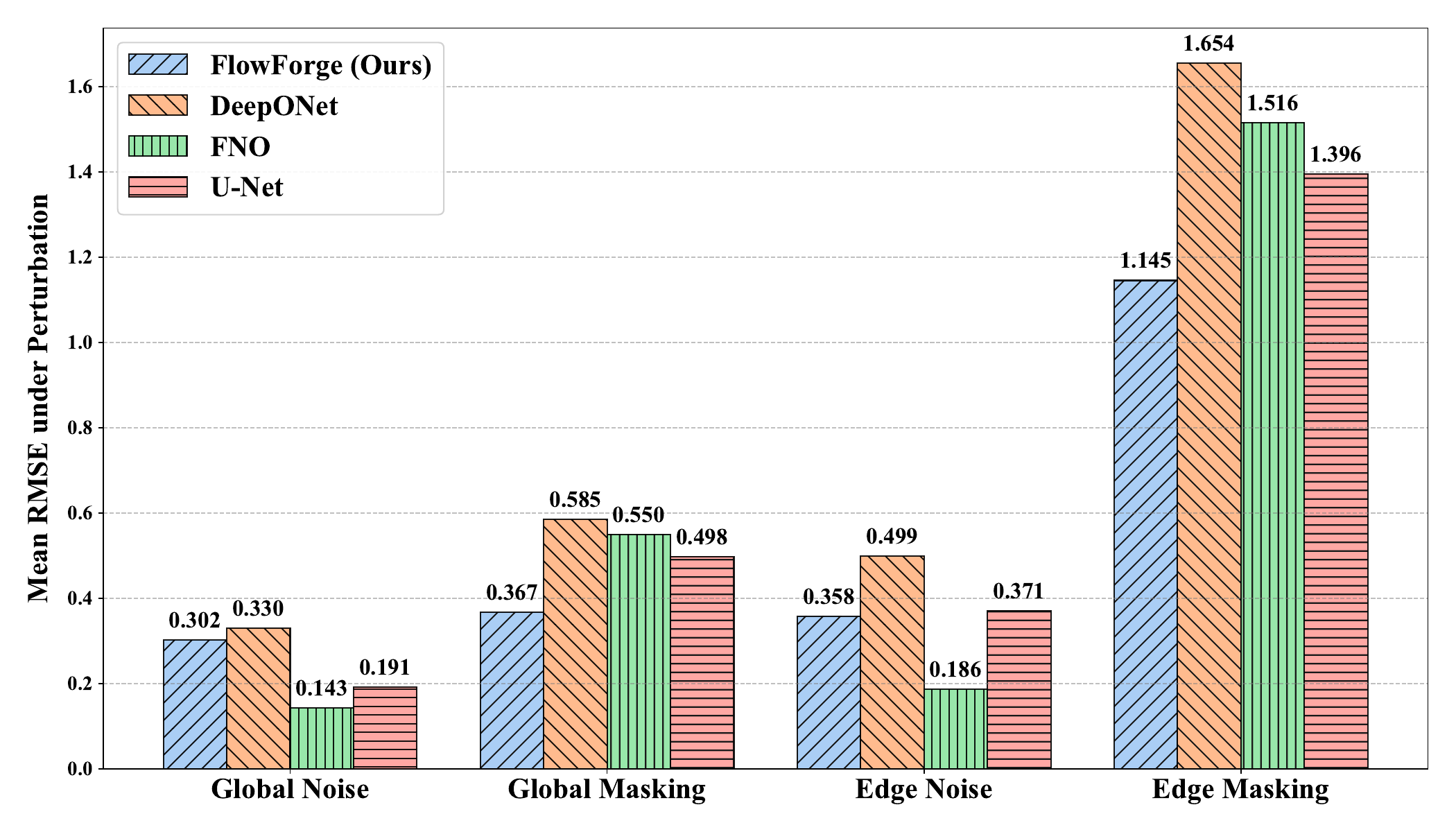}
    \end{minipage}

    \par\medskip

    \begin{minipage}[t]{0.48\linewidth}
        \centering
        \textbf{CFDBench-Tube}\par\smallskip
        \includegraphics[width=\linewidth]{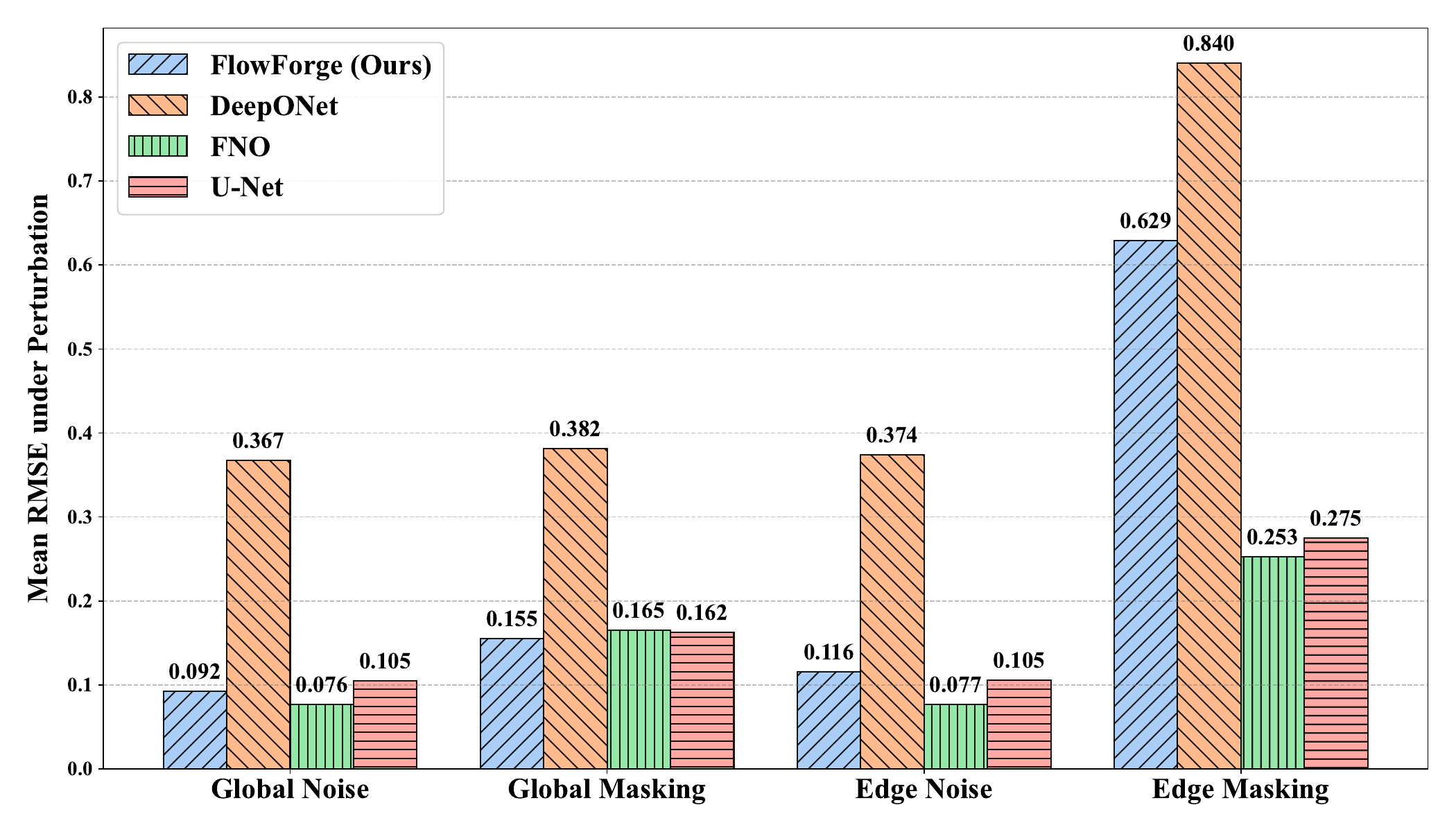}
    \end{minipage}\hfill
    \begin{minipage}[t]{0.48\linewidth}
        \centering
        \textbf{CFDBench-Dam}\par\smallskip
        \includegraphics[width=\linewidth]{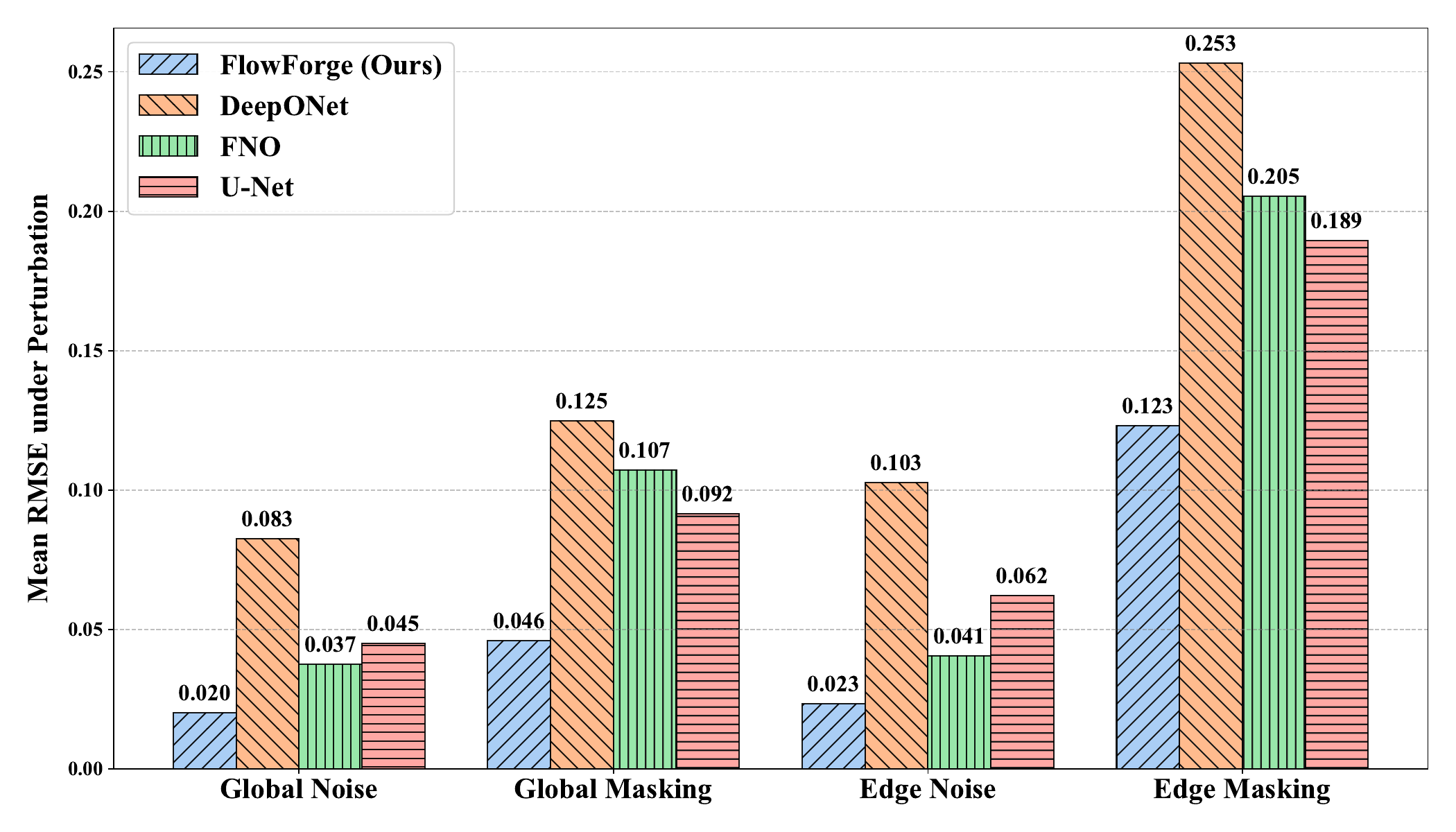}
    \end{minipage}

    \captionsetup{hypcap=false}
    \captionof{figure}{Complete robustness summaries across all four CFDBench datasets.
    Each panel aggregates performance under global and edge-localized
    perturbations across the noise, masking, and boundary-corruption protocols
    described in Section~\ref{sec:robustness}. \sys remains consistently more
    stable than the baselines throughout the suite.}
    \label{fig:appendix_robustness_grid}
    \par
\endgroup

\clearpage
\section{Extended Case Studies and Qualitative Analysis}
\label{app:extended_case_studies}

In Section~\ref{sec:accuracy_analysis}, we quantitatively analyzed the flow dynamics on the \texttt{FB-Gravity} dataset. To demonstrate the versatility of \sys across diverse flow regimes and verify its stability under corruption, we provide additional qualitative comparisons here. We focus on two complementary aspects: (1) prediction fidelity in complex multiphase flows involving sharp interfaces and phase changes, and (2) visual verification of error containment mechanisms under severe input degradation.

\subsection{Fidelity on Multiphase Benchmarks}

We examine two challenging scenarios from the CFDBench and BubbleML suites, characterized by rapid topological changes and multi-scale interfacial dynamics.

\textbf{Dam Break (Multiphase VOF).} \autoref{fig:case_dam} visualizes the Volume-of-Fluid (VOF) simulation of a collapsing fluid column. This regime poses a significant challenge due to the presence of sharp gas-liquid interfaces and the generation of fine-scale vorticity during impact.

To scrutinize performance in physically volatile regions, the visualization combines macroscopic velocity streamlines with a microscopic error analysis. We define a region of intense vorticity gradients (marked by the green bounding box) and overlay the absolute vorticity error heatmap ($|\omega_{\text{pred}} - \omega_{\text{gt}}|$) to reveal localized discrepancies.

As observed in \autoref{fig:case_dam}, the performance gap manifests differently across model families. U-Net and DeepONet exhibit substantial deviations, marked by intense red hotspots in the error heatmaps, indicating a struggle to capture the complex multi-fluid mixing dynamics. The comparison between \sys and FNO is more nuanced yet physically critical. While FNO achieves a low global error magnitude comparable to \sys (indicated by similarly faint heatmaps), a closer inspection of the streamlines reveals a key divergence: FNO suffers from \textit{over-smoothing} at the vortex core, failing to resolve the tight rotational structures. In contrast, \sys precisely preserves these fine-grained dynamics. This confirms that the staged local rollout does not merely lower pixel-wise error but, more importantly, prevents the loss of high-frequency physical information that plagues monolithic global predictors.

\textbf{Flow Boiling (FB-VelScale).}
\autoref{fig:case_inletvel} depicts the void fraction field in a heated channel flow. This regime involves small-scale bubble nucleation and advection, governed by the interplay between buoyancy and inlet velocity perturbations.

We highlight two distinct regions to illustrate scale-dependent performance.
\begin{itemize}
    \item \textbf{Inlet Region (Box 1):} Located near the channel entrance, flow dynamics are relatively smooth and dominated by inlet boundary conditions. \sys successfully captures this laminar behavior. In contrast, U-Net and FNO introduce spurious oscillations, leading to significant local errors.
    \item \textbf{Nucleation Region (Box 2):} Located near the heated wall, this area features vigorous bubble formation and detachment, characterized by high-frequency spatial variations in void fraction. Here, \sys effectively preserves individual bubble structures. Baselines, particularly FNO and U-Net, exhibit a washout effect---blurring adjacent bubbles into a unified field due to the spectral bias of global operators---which is detrimental for accurate heat transfer estimation.
\end{itemize}

\subsection{Visualizing Robustness Mechanisms}

To complement the aggregate robustness scores in Section~\ref{sec:robustness}, we visualize the flow fields under two distinct corruption protocols to empirically validate the error containment hypothesis.

\textbf{Global Block Masking.}
We apply a random binary block mask to the input frame of the \texttt{Cavity} dataset, zeroing out approximately 10\% of the spatial domain (\autoref{fig:robust_cavity}).

For global architectures like U-Net and FNO, the missing blocks introduce high-frequency discontinuities that propagate globally via convolution or Fourier mixing. This results in spectral poisoning: artifacts teleport from the masked regions to the unmasked center and distort the primary vortex structure. The rectangular regions of high error visible in the baseline heatmaps (\autoref{fig:robust_cavity}) directly correspond to this non-local error propagation.

In contrast, due to strict \textit{local causality}, \sys ensures that errors at a masked site can only influence its immediate topological neighbors in the subsequent rollout stage. Consequently, the majority of the flow field remains structurally intact, demonstrating that \sys effectively confines local defects, preventing them from destroying global physical consistency.

\textbf{Global Gaussian Noise Injection.}
We inject additive Gaussian noise ($\mathcal{N}(0, 0.1^2)$) into some random blocks of the input velocity field for the \texttt{Cylinder} dataset (\autoref{fig:robust_cylinder}).

Global models tend to amplify high-frequency noise, interpreting it as physical turbulence. This results in disordered streamlines around the cylinder and a distorted wake structure, as shown in the baseline predictions. \sys, however, leverages its local aggregation mechanism as an implicit spatial filter. By conditioning the prediction on a bounded neighborhood of previous estimates, \sys effectively averages out the incoherent zero-mean noise, recovering a clean Von Karman vortex street that closely matches the ground truth.

\begin{figure}[htbp]
    \centering
    \includegraphics[width=\buildfigurewidth{\linewidth}]{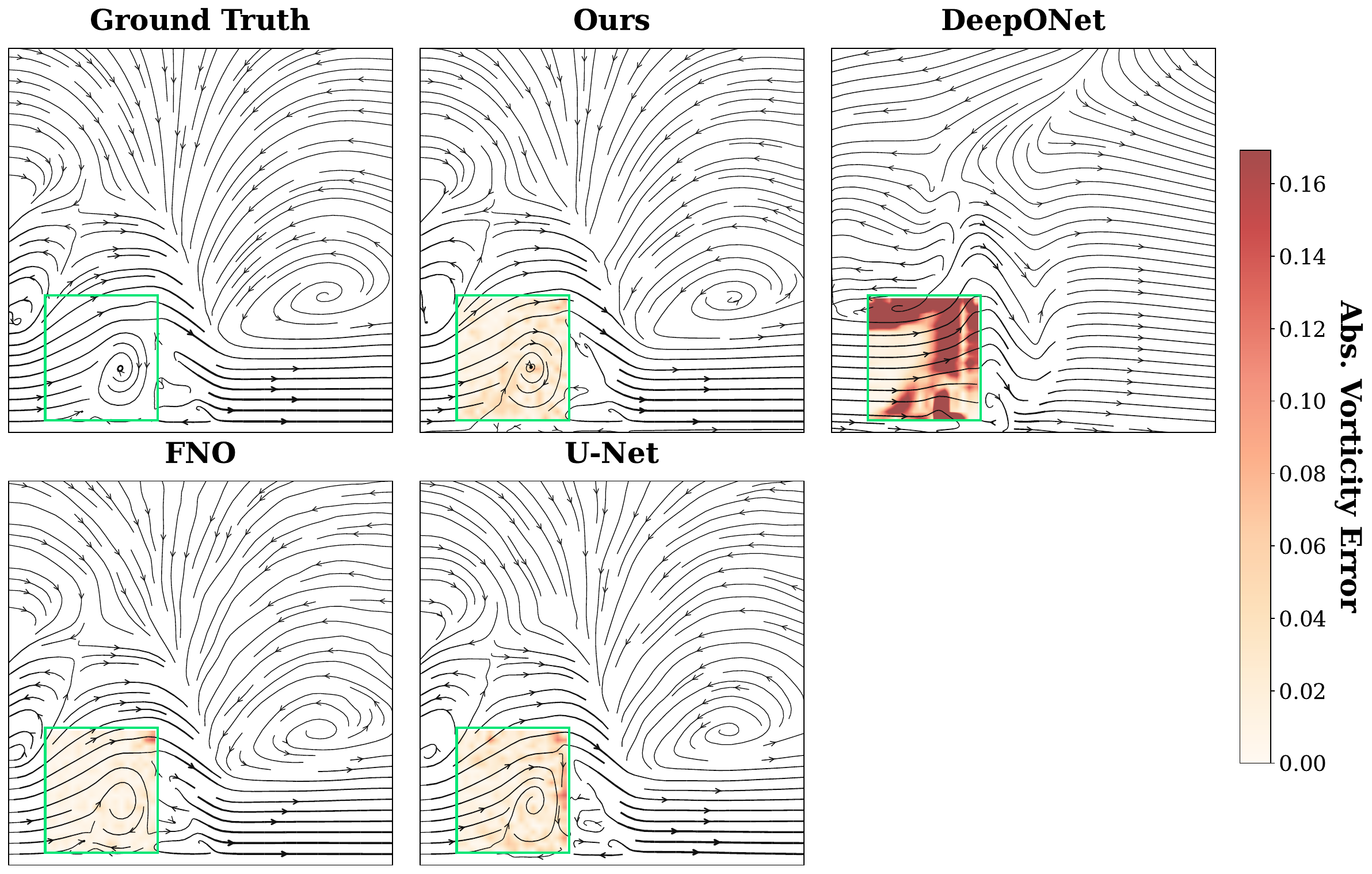}
    \caption{\textbf{Qualitative results on CFDBench-Dam.} Visualization of the fluid phase fraction and vorticity error. While FNO tends to over-smooth tight vortex cores, \sys accurately captures the sharp interface and rotational dynamics of the water column as it impacts the obstacle.}
    \label{fig:case_dam}
\end{figure}

\begin{figure}[htbp]
    \centering
    \includegraphics[width=\buildfigurewidth{\linewidth}]{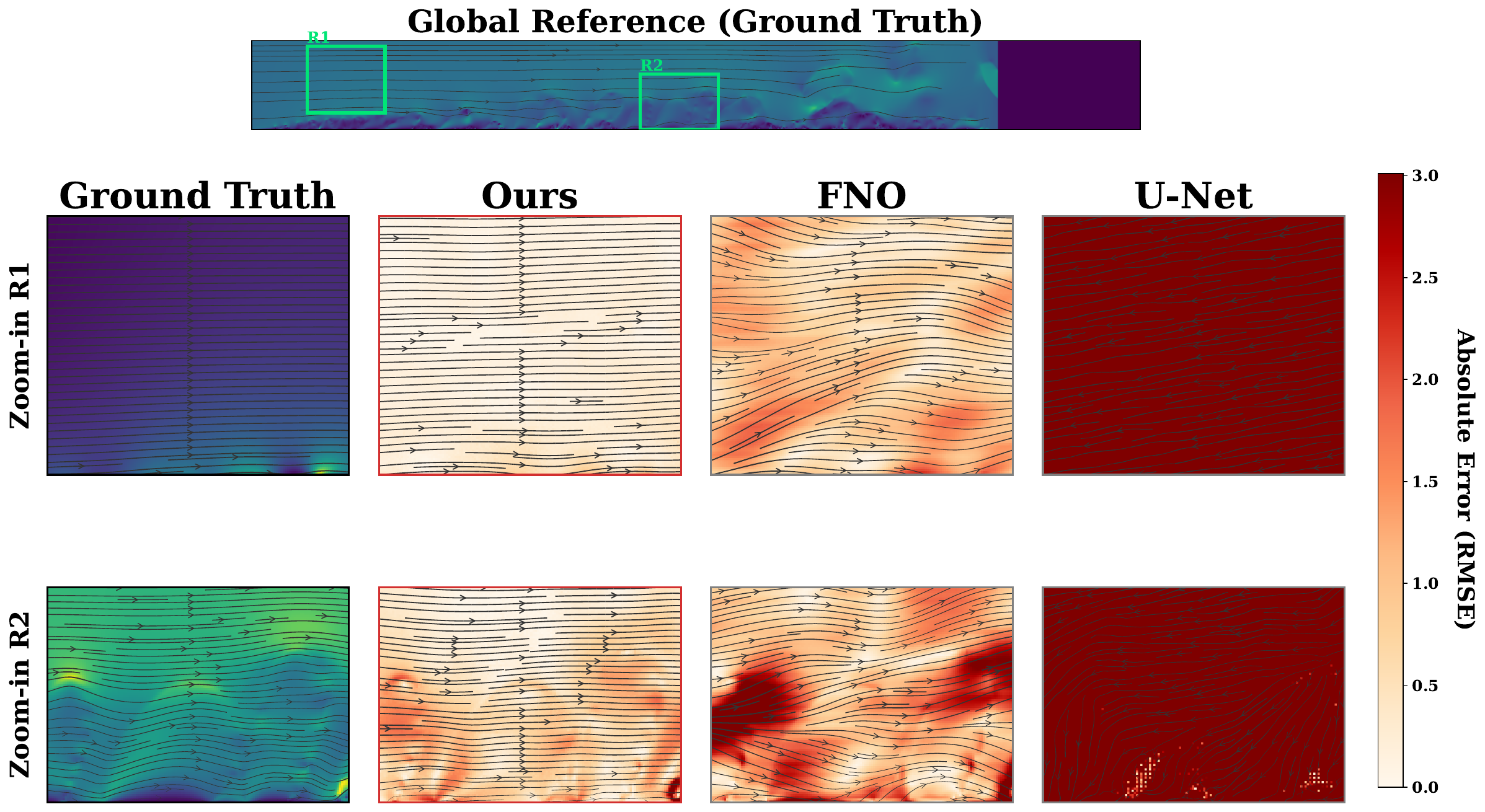}
    \caption{\textbf{Qualitative results on BubbleML FB-VelScale.} Visualization of gas phase distribution (void fraction). \sys preserves the integrity of small-scale bubbles (Box 2) that are prone to washout in global baselines, while maintaining stability in smooth inlet regions (Box 1).}
    \label{fig:case_inletvel}
\end{figure}

\begin{figure}[htbp]
    \centering
    \includegraphics[width=\buildfigurewidth{\linewidth}]{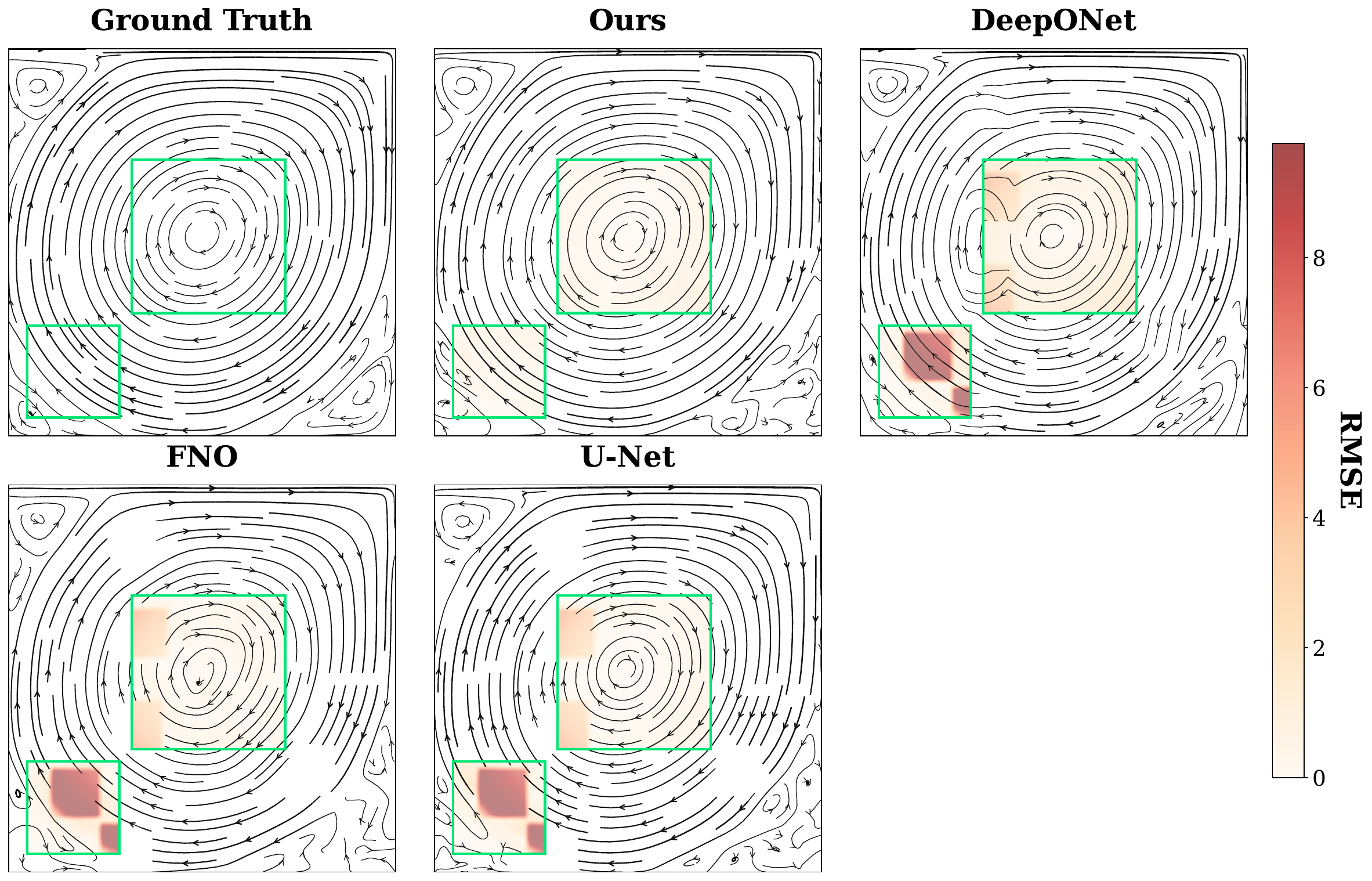}
    \caption{\textbf{Robustness to Global Block Masking (Cavity, $\approx$10\% masked).} The input contains zeroed-out blocks. Baselines exhibit non-local error propagation (artifacts appearing far from masks), whereas \sys strictly confines errors to the masked regions, preserving the central vortex details.}
    \label{fig:robust_cavity}
\end{figure}

\begin{figure}[htbp]
    \centering
    \includegraphics[width=\buildfigurewidth{\linewidth}]{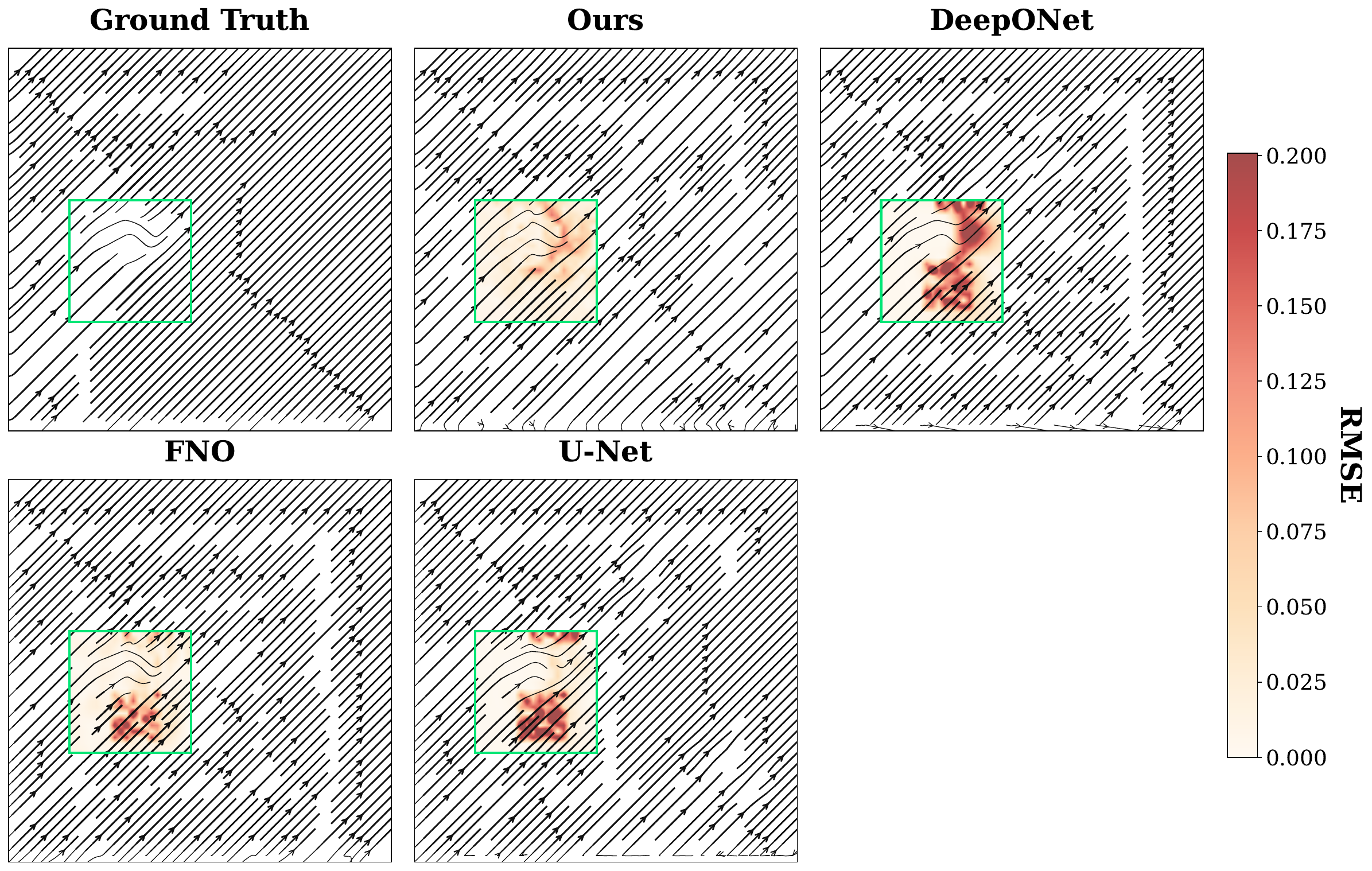}
    \caption{\textbf{Robustness to Localized Gaussian Noise (Cylinder, $\sigma=0.1$).} Comparison of streamlines under noisy input. Global models amplify noise into spurious turbulence. \sys acts as a conditional filter, recovering a smooth and physically consistent wake.}
    \label{fig:robust_cylinder}
\end{figure}

\clearpage


\section{Additional Theoretical Details and Proofs}
\label{sec:proofs}

\subsection{Notation and Standing Assumptions}
\begin{itemize}
  \item \textbf{Domain and norms.} The spatial domain is either the $d$-torus
  $\Omega=\mathbb{T}^d$ (periodic) or a bounded Lipschitz domain with homogeneous
  Neumann (no-flux) boundary conditions. We write
  $\langle f,g\rangle=\int_{\Omega} f\cdot g$ and $\|f\|_2^2=\langle f,f\rangle$.

  \item \textbf{Discrete field.} The next-step field is represented on a uniform grid as
  $U_{t+1}\in\mathbb{R}^{N\times c}$ (or equivalently vectorized in $\mathbb{R}^{Nc}$),
  with grid spacing $\Delta x$ per axis.

  \item \textbf{Ordering and stages.} An \emph{ordering scheme} induces an order $\sigma$
  of the spatial indices $\{1,\ldots,N\}$ and a stage partition $S_1,\ldots,S_K$ (disjoint union).

  \item \textbf{Stage-wise rollout map.} At inference, stage $s$ outputs
  \[
    Y_{S_s} \;=\; G_\theta\!\big(U_t,\;Y_{S_{<s}}\big),\qquad S_{<s}:=\{S_1,\ldots,S_{s-1}\}.
  \]
  The predictor $G_\theta$ is Fr\'echet differentiable in its second argument and is \emph{masked}:
  it does not access $Y_{S_s}$ nor any $Y_{S_r}$ with $r>s$.

  \item \textbf{Norms for block vectors.} For block vectors $z=(z_1,\ldots,z_m)$ we use
  the Euclidean norm $\|\cdot\|$ and the induced operator norm for Jacobians.

  \item \textbf{Helmholtz projection (incompressible case).} Let
  \[
    \mathcal{V}=\{\,v\in L^2(\Omega;\mathbb{R}^d)\;:\;\nabla\!\cdot v=0,\;
    v\!\cdot n|_{\partial\Omega}=0 \text{ (if } \partial\Omega\neq\emptyset)\,\},
  \]
  and let $P:L^2(\Omega;\mathbb{R}^d)\to \mathcal{V}$ denote the $L^2$-orthogonal projection
  (Helmholtz--Hodge projection): $Pv=v-\nabla \phi$ where $\phi$ solves
  $\Delta \phi=\nabla\!\cdot v$ with compatible BCs (and zero-mean in the periodic case).
\end{itemize}

\subsection{Causality from Masking: Triangular Jacobian}
\label{app:triangular}
\subsubsection{Block lower--triangular context Jacobian}

\begin{theorem}[Block lower--triangular context Jacobian]\label{thm:triangular}
Define $F:\mathbb{R}^{Nc}\to\mathbb{R}^{Nc}$ blockwise by
\[
\begin{aligned}
F_{S_1}(Y) &= G_\theta(U_t,\emptyset),\\
F_{S_2}(Y) &= G_\theta(U_t,Y_{S_1}),\\
&\ \,\vdots\\
F_{S_s}(Y) &= G_\theta(U_t,Y_{S_{<s}}).
\end{aligned}
\]
Let $J^{\mathrm{ctx}} := \frac{\partial F}{\partial Y}$ in the stage-block layout. Then
\[
J^{\mathrm{ctx}}_{s,r} \;=\; \frac{\partial F_{S_s}}{\partial Y_{S_r}} \;=\; 0 \qquad \text{for all } r\ge s,
\]
i.e., $J^{\mathrm{ctx}}$ is strictly block lower--triangular. In particular, $(J^{\mathrm{ctx}})^K=0$.
\end{theorem}

\begin{proof}
By construction, $F_{S_s}$ depends only on $(Y_{S_1},\ldots,Y_{S_{s-1}})$. Hence for any $r\ge s$, $F_{S_s}$ is independent of $Y_{S_r}$ and the Fr\'echet partial derivative with respect to $Y_{S_r}$ is the zero operator, giving $J^{\mathrm{ctx}}_{s,r}=0$. Strict lower--triangular block matrices are nilpotent with index at most $K$.
\end{proof}

\begin{corollary}[One-way error flow]\label{cor:oneway}
A perturbation injected into $Y_{S_r}$ cannot affect any $Y_{S_q}$ with $q\le r$. Error can only propagate to later stages.
\end{corollary}

\begin{proof}
Immediate from $J^{\mathrm{ctx}}_{q,r}=0$ for $q\le r$.
\end{proof}

\subsection{Stage Sensitivity and Error Propagation}\label{app:stability}
\subsubsection{Lipschitz product bound}

\begin{assumption}[Stage Lipschitzness]\label{assump:lipschitz}
For each stage $s$ there exists $L_s\ge 0$ such that for all context vectors $a,b$,
\[
\big\|G_\theta(U_t,a)-G_\theta(U_t,b)\big\| \;\le\; L_s\,\|a-b\|.
\]
\end{assumption}

\begin{proposition}[Geometric damping of upstream perturbations]\label{prop:prodL}
Fix $r<s$. Consider the composite mapping from $Y_{S_r}$ to $Y_{S_s}$ induced by executing stages $r+1,\ldots,s$ with all other inputs fixed. Then
\[
\left\|\,\frac{\partial\,Y_{S_s}}{\partial\,Y_{S_r}}\,\right\|
\;\le\; \prod_{q=r+1}^{s} L_q .
\]
Equivalently, any small perturbation $\delta Y_{S_r}$ obeys
\[
\|\delta Y_{S_s}\| \;\le\; \Big(\prod_{q=r+1}^{s} L_q\Big)\,\|\delta Y_{S_r}\|.
\]
\end{proposition}

\begin{proof}
Define $H_{r+1}(z)=G_{r+1}(U_t,(Y_{S_{<r}},z))$, which is $L_{r+1}$-Lipschitz by Assumption~\ref{assump:lipschitz}. Similarly define $H_{r+2},\ldots,H_s$ for successive stages. The dependence from $Y_{S_r}$ to $Y_{S_s}$ is the composition $Y_{S_s} = H_s\!\circ\cdots\circ H_{r+1}(Y_{S_r})$. Lipschitz constants multiply under composition, yielding the bound.
\end{proof}

\begin{corollary}[Monotone decay under easier later stages]\label{cor:monotone}
If $L_{q+1}\le L_q$ and often $L_q<1$, then $\prod_{q=r+1}^{s}L_q$ decays geometrically in $s-r$. A widening schedule (smaller $K$) further reduces worst-case propagation depth.
\end{corollary}

\subsubsection{Stage-wise attenuation (sufficient-condition analysis)}

\begin{theorem}[Stage-wise error attenuation under contraction]\label{thm:selfhealing}
Let $y^\star_{S_s}$ be the predictions obtained when earlier contexts equal their ground-truth values, and let $\hat y_{S_s}$ be the actual predictions obtained when earlier stages may carry errors. Suppose Assumption~\ref{assump:lipschitz} holds and $\max_s L_s\le \lambda<1$. Then for every $s$,
\[
\big\|\hat y_{S_s}-y^\star_{S_s}\big\|
\;\le\;
\sum_{r=1}^{s-1}\Big(\prod_{q=r+1}^{s} L_q\Big)\;\big\|\hat y_{S_r}-y^\star_{S_r}\big\|.
\]
In particular,
\[
\max_{j\le s}\big\|\hat y_{S_j}-y^\star_{S_j}\big\|
\;\le\;
\frac{\lambda}{1-\lambda}\;\max_{j< s}\big\|\hat y_{S_j}-y^\star_{S_j}\big\|.
\]
Thus newly generated tokens attenuate prior errors, and total error cannot blow up across stages.
\end{theorem}

\begin{proof}
By definition,
\(
\hat y_{S_s}-y^\star_{S_s}
=
G_\theta(U_t,\hat y_{S_{<s}})-G_\theta(U_t,y^\star_{S_{<s}}).
\)
Assumption~\ref{assump:lipschitz} and the triangle inequality imply
$\|\hat y_{S_s}-y^\star_{S_s}\|\le L_s\,\|\hat y_{S_{<s}}-y^\star_{S_{<s}}\|$.
Expand block-by-block and apply Proposition~\ref{prop:prodL}; then bound the geometric series by $\lambda/(1-\lambda)$.
\end{proof}

\subsection{Short-Horizon Locality and Coverage Bounds}\label{app:locality}

\subsubsection{Kernel form for advection--diffusion}

\begin{lemma}[Advection--diffusion fundamental solution]\label{lem:fundamental}
Consider the scalar advection--diffusion PDE
\[
\partial_t u + a\!\cdot\!\nabla u \;=\; \nu\,\Delta u,\qquad a\in \mathbb{R}^d,\ \nu>0,
\]
with sufficiently regular $u(\cdot,t)$ and periodic or no-flux BCs. Then
\[
u(x,t+\Delta t)
\;=\;
\int_{\mathbb{R}^d} G_{\Delta t}(y)\; u\big(x - a\,\Delta t - y,\;t\big)\,dy,
\]
\[
G_{\Delta t}(y)
\;=\;
\frac{1}{(4\pi\nu\Delta t)^{d/2}}\,
e^{-\frac{\|y\|^2}{4\nu\Delta t}}.
\]
\end{lemma}

\begin{proof}
  Let $w(x,t)=u(x+a t,t)$. Then $\partial_t w = a\!\cdot\!\nabla u + \partial_t u = \nu\,\Delta u = \nu\,\Delta w$. Hence $w_t=\nu\Delta w$ and $w(\cdot,t+\Delta t)=G_{\Delta t}*w(\cdot,t)$. Rewriting in terms of $u$ yields the claim.
\end{proof}
 
\subsubsection{Tail mass inside a radius}

\begin{lemma}[Mass inside a radius]\label{lem:mass}
For $Y\sim \mathcal{N}(0,2\nu\Delta t\,I_d)$,
\[
\mathbb{P}\big(\|Y\|\le r\big)
\;=\;
F_{\chi^2_d}\!\left(\frac{r^2}{2\nu\Delta t}\right),
\]
where $F_{\chi^2_d}$ is the CDF of the $\chi^2$ distribution with $d$ degrees of freedom.
\end{lemma}

\begin{proof}
$\|Y\|^2/(2\nu\Delta t)\sim \chi^2_d$ by isotropy.
\end{proof}

\subsubsection{Tail-mass error bound for truncated dependence}

\begin{proposition}[CFL-style coverage $\Rightarrow$ tail-mass error bound]\label{prop:cfl}
Fix $p\in(0,1)$ and let
\[
r_p \;:=\; \sqrt{\,2\nu\Delta t\;F_{\chi^2_d}^{-1}(p)\,}
\]
so that $\int_{\|y\|\le r_p} G_{\Delta t}(y)\,dy = p$. Define the truncated predictor
\[
v(x)
\;:=\;
\int_{\|y\|\le r} G_{\Delta t}(y)\; u\big(x-a\Delta t-y,t\big)\,dy.
\]
Then for any $r\ge r_p$:
\begin{align*}
\text{(i) } & \|u(\cdot,t+\Delta t)-v\|_\infty \ \le\ (1-p)\,\|u(\cdot,t)\|_\infty,\\
\text{(ii) } & \|u(\cdot,t+\Delta t)-v\|_2 \ \le\ (1-p)\,\|u(\cdot,t)\|_2.
\end{align*}
\end{proposition}

\begin{proof}
By Lemma~\ref{lem:fundamental}, the error equals the kernel tail
$e(x)=\int_{\|y\|>r} G_{\Delta t}(y)\,u(x-a\Delta t-y,t)\,dy$.
Bounding by $\|u(\cdot,t)\|_\infty$ gives (i) since $\int_{\|y\|>r}G_{\Delta t}=1-p$ for $r\ge r_p$.
For (ii), Young's inequality yields
$\|e\|_2 \le \|G_{\Delta t}\mathbf{1}_{\{\|y\|>r\}}\|_1 \|u(\cdot,t)\|_2 \le (1-p)\|u(\cdot,t)\|_2$.
\end{proof}

\subsubsection{Coverage under Outward Spiral traversal (geometry)}

Let the Outward Spiral include all grid points within Chebyshev radius $R$ (i.e., $\ell_\infty$-ball $B_\infty(R\Delta x)$). The Euclidean and Chebyshev balls satisfy
\begin{equation}\label{eq:norm-inclusions}
B_2(r)\ \subseteq\ B_\infty(r)\ \subseteq\ B_2(\sqrt{d}\,r).
\end{equation}

\begin{lemma}[Ring coverage vs.\ Euclidean radius]\label{lem:ring-coverage}
If the Outward Spiral has generated $R$ rings, then $B_2(r)\subseteq$ (available context) for all $r\le R\Delta x$. Conversely, if $B_2(r)$ is required, it suffices to take $R\ge \lceil r/\Delta x\rceil$ rings.
\end{lemma}

\begin{proof}
The available context equals $B_\infty(R\Delta x)$; apply the left inclusion in \eqref{eq:norm-inclusions}.
\end{proof}

\begin{definition}[Coverage coefficient on a grid]\label{def:coverage}
Let the Outward Spiral have generated $R$ rings prior to decoding the token at $x$. Define
\[
\Gamma_p(x)\;:=\;\frac{R\,\Delta x}{\;\|a\|\,\Delta t+\sqrt{\,2\nu\Delta t\,F_{\chi^2_d}^{-1}(p)\,}\;}.
\]
By Lemma~\ref{lem:ring-coverage}, $\Gamma_p(x)\ge 1$ implies $B_2(r_p)\subseteq$ (available context).
\end{definition}

\begin{remark}[Discrete coverage with norm equivalence]\label{rem:discrete}
If the implementation reasons in the $\ell_\infty$ geometry, one may also use the inclusion $B_2(r)\subseteq B_\infty(r)$;
thus $\Gamma_p(x)\ge 1$ (Euclidean) is conservative for Chebyshev coverage.
Combining Proposition~\ref{prop:cfl} with Lemma~\ref{lem:quadrature} yields a practical bound
\[
\|u(\cdot,t+\Delta t)-v_{\mathrm{grid}}\|_{\infty,2} \ \le\ (1-p)\,\|u(\cdot,t)\|_{\infty,2}\;+\;C\,\Delta x .
\]
\end{remark}

\subsubsection{Discrete quadrature error for truncated Gaussian}
\begin{lemma}[Discrete quadrature error for truncated Gaussian]\label{lem:quadrature}
Assume $u(\cdot,t)\in C^1(\Omega)$ and use a Riemann sum over available grid points in $B_2(r)$ to approximate $v(x)$.
Then, uniformly in $x$,
\[
\big|\, v(x) - v_{\mathrm{grid}}(x)\,\big| \;\le\; C\,\Delta x,
\]
where $C$ depends on $\|u(\cdot,t)\|_{C^1}$, $\nu$, $\Delta t$, and $d$ (Gaussian smoothness).
\end{lemma}

\begin{proof}
The truncated Gaussian is smooth and rapidly decaying; a standard first-order Riemann-sum estimate for Lipschitz integrands gives $O(\Delta x)$ uniform error.
\end{proof}

\subsubsection{Hyperbolic and multiphysics extensions}
For a linear symmetric hyperbolic system $\partial_t w + \sum_{j=1}^d A_j \partial_{x_j} w=0$ with wave speed bound $c_{\max}$, the domain of dependence over $\Delta t$ is contained in $B_2(c_{\max}\Delta t)$.
For advection--diffusion--heat (with thermal diffusivity $\kappa$), a conservative requirement is
\[
r_{\mathrm{req}} \;\approx\; c_{\max}\Delta t \;+\; \sqrt{2\nu\Delta t\,F_{\chi^2_d}^{-1}(p)} \;+\; \sqrt{2\kappa\Delta t\,F_{\chi^2_d}^{-1}(p)} .
\]
This can be used in Definition~\ref{def:coverage} by replacing $\|a\|$ with $c_{\max}$ and augmenting the diffusive radii.

\subsection{Non-expansiveness and persistence of coverage bound}

We now make precise that the coverage bound is preserved by the pressure projection.

\begin{lemma}[Helmholtz projection is non-expansive]\label{lem:projection-nonexp}
$P:L^2(\Omega;\mathbb{R}^d)\to L^2(\Omega;\mathbb{R}^d)$ is an $L^2$-orthogonal projection onto $\mathcal{V}$.
Hence $\|Pv-Pw\|_2 \le \|v-w\|_2$ and, in particular, $\|Pv\|_2\le \|v\|_2$.
Moreover, there exists a domain-dependent constant $C_\Omega$ such that
\[
\|(I-P)v\|_2 = \|\nabla \phi\|_2 \;\le\; C_\Omega\,\|\nabla\!\cdot v\|_2,
\]
where $\Delta \phi=\nabla\!\cdot v$ (elliptic regularity).
\end{lemma}

\begin{proof}
Orthogonality is standard: decompose $v=w+\nabla\phi$ with $w\in\mathcal{V}$ and $\phi$ solving the Poisson problem; $P$ is the orthogonal projector $v\mapsto w$ so it is non-expansive in $L^2$.
The estimate $\|\nabla\phi\|_2 \le C_\Omega \|\Delta \phi\|_{H^{-1}}\le C_\Omega\|\nabla\!\cdot v\|_2$ follows from elliptic regularity on $\Omega$ with the stated BCs.
\end{proof}

\begin{proposition}[Coverage bound persists under projection]\label{prop:cfl-projected}
Let $u^\ast(\cdot,t+\Delta t)$ denote the advective--diffusive update and let the true incompressible update be $u(\cdot,t+\Delta t)=P u^\ast(\cdot,t+\Delta t)$.
Let $v$ be the truncated predictor from Proposition~\ref{prop:cfl} and $\hat u:=Pv$.
Then
\[
\|u(\cdot,t+\Delta t)-\hat u\|_2
\;=\; \|P u^\ast - P v\|_2
\;\le\; \|u^\ast - v\|_2
\;\le\; (1-p)\,\|u(\cdot,t)\|_2 .
\]
In particular, the projection step \emph{cannot worsen} the coverage error bound; it can only improve it.
\end{proposition}

\begin{proof}
Apply Lemma~\ref{lem:projection-nonexp} to $v$ and $u^\ast$, then Proposition~\ref{prop:cfl}(ii).
\end{proof}

\subsection{Runtime Model: Idealized Makespan}\label{app:makespan}

\begin{proposition}[Latency as a tunable knob]\label{prop:makespan}
Assume each token evaluation takes constant time $t_{\mathrm{tok}}$, tokens within a stage are independent, and up to $w_s$ tokens of stage $s$ can be processed in parallel. Then the stage makespan satisfies
\[
T_s \;=\; \Big\lceil \frac{|S_s|}{w_s} \Big\rceil\, t_{\mathrm{tok}},
\]
and the end-to-end latency is
\[
T_{\mathrm{total}}
\;=\;\sum_{s=1}^{K} T_s
\;=\;\sum_{s=1}^{K}\Big\lceil \frac{|S_s|}{w_s} \Big\rceil\, t_{\mathrm{tok}}.
\]
Hence, with fixed $t_{\mathrm{tok}}$, latency is directly controlled by the stage widths $\{w_s\}$.
\end{proposition}

\begin{proof}
Within stage $s$, $|S_s|$ identical jobs on $w_s$ identical parallel servers have optimal makespan $\lceil |S_s|/w_s\rceil t_{\mathrm{tok}}$ (tight list-scheduling bound for identical tasks). Stages are executed sequentially, so the total time is the sum.
\end{proof}

\end{document}